\pdfoutput=1  
\documentclass[letterpaper, 10 pt, conference]{ieeeconf}  

\IEEEoverridecommandlockouts                              

\usepackage{graphicx} 
\usepackage{mathptmx} 
\usepackage{times} 
\usepackage{amsmath}
\usepackage{amssymb}
\usepackage{booktabs}
\usepackage{multirow}
\usepackage{multicol}
\usepackage{array}
\usepackage{xcolor}
\usepackage{cuted}
\usepackage{url}
\usepackage{pifont}
\usepackage{subcaption}
\usepackage[labelfont=bf, labelsep=period]{caption}
\let\labelindent\relax 
\usepackage{enumitem}
\usepackage{amsthm}

\newtheorem{theorem}{Theorem}
\newtheorem{corollary}{Corollary}
\theoremstyle{definition}

\newenvironment{flushitemize}
  {\begin{itemize}[leftmargin=*, labelindent=0pt]}
  {\end{itemize}}

\usepackage{multibib}
\newcites{app}{Appendix References}

\usepackage{listings}
\usepackage{mdframed}

\newcommand{\red}[1]{}

\newcommand{\ourMethod}[0]{GAP}

\usepackage[hidelinks,breaklinks=true]{hyperref}
\hypersetup{
  pdftitle={Robots That Take Initiative: A Framework for Building and Evaluating Proactive Robots},
  pdfauthor={Maithili Patel, Sonia Chernova},
  pdfsubject={Human-Robot Interaction},
  pdfkeywords={proactive robot assistance, human-robot interaction, closed-loop evaluation, goal anticipation}
}

\title{\LARGE \bf Robots That Take Initiative: A Framework for \\ Building and Evaluating Proactive Robots
\vspace{-3mm}
}

\author{Maithili Patel$^{1}$, Sonia Chernova$^{1}$%
\thanks{$^{1}$Both authors are with School of Interactive Computing, Georgia Institute of Technology, USA
        {\tt\footnotesize maithili@gatech.edu}}%
}

\begin{document}

\maketitle


\pagestyle{plain}

%
\begin{abstract}
Effective robot assistance beyond narrow roles and repetitive tasks requires robots to be \textit{proactive} -- to decide what needs to be done rather than waiting to be told. While proactivity is increasingly explored, it lacks a unified formulation, and work in the domain is typically evaluated offline against static human models that cannot capture the effect of a robot's actions on the environment and the user's own behavior. We introduce a unified formalism for proactive robot assistance, organize it into three levels, and provide a framework to address the highest level of unprompted proactive assistance. We then show that offline evaluation overstates performance in this setting, and contribute a closed-loop evaluation with a human model that adapts to the robot. Finally, we present a method, GAP, that instantiates our framework, learning from passive observation to anticipate user goals and act. Under closed-loop evaluation, prior state-of-the-art methods collapse, in some cases adding more work than they save, while GAP remains robust and substantially outperforms them\footnote{Code for the evaluation setup and GAP model: \url{https://github.com/Maithili/GAP}}.
\end{abstract}

\section{Introduction}
Robots are becoming increasingly capable, moving beyond narrow, repetitive jobs toward general-purpose assistance in everyday life. Home and service robots are now being developed to operate in unstructured environments, from commercial humanoids designed for household chores to research systems for domestic mobile manipulation~\cite{nasiriany2026robocasa365, shafiullah2023bringing}. Unlike factory robots, these systems must operate with little supervision, handling tasks such as laundry, tidying, or meal preparation while users are away or occupied.


This shift toward capable robots operating in unstructured settings with little supervision reframes robotics as not only \textit{how} a robot performs a task, but \textit{what} it should do and \textit{when}. Most current approaches assume a human instructs the robot on what to do~\cite{brohan2023can, xiao2025robi}, but instructing every action does not scale. Tasking a robot can itself be burdensome~\cite{daminger2019cognitive}, and user studies find that people would rather robots learn their habits from observation than receive constant instruction~\cite{acknowledge}. Effective human assistance stems from taking initiative rather than waiting for instructions. Studies of domestic work describe the ideal helper as someone who identifies needs and acts without being asked~\cite{elden2019nanny}, while teamwork research finds that high-performing teams coordinate implicitly, anticipating and adapting to each other's needs rather than relying on explicit instruction~\cite{rico2008team}. In both cases, the key skill is not following orders, but recognizing what needs to be done. In other words, users want robots to be \textit{proactive}.

Despite this interest, research toward proactivity remains fragmented. Existing works address different versions of the problem under varying assumptions, without a unified framework connecting them.
A common assumption, especially when no explicit goal is given, is to evaluate against a static human model. Such offline evaluation ignores how robot actions reshape the environment and the user's own behavior, rewarding policies that look strong offline but degrade once real users adapt. Consequently, two key open problems remain: formalizing proactivity as a whole and evaluating proactive behavior in closed loop. Solving them is critical for understanding existing approaches and revealing what challenges remain unsolved.

In this work, we contribute the first unified formalism and three levels of proactivity for robotic assistance. We focus on the highest level, \textit{unprompted proactive assistance}, where the robot is given no goal or input, as it draws focus to the challenges of unstructured settings. Along with a detailed problem formulation, we decompose the problem into algorithmic sub-problems, and situate existing approaches, showing how they address a scoped version of the problem.

Second, we demonstrate that existing methods addressing long-horizon, goal-free proactivity~\cite{patel_slatepro_2023, bartoli_streak_2025} share a deeper limitation of being evaluated through static datasets or human models. Static user models, which do not adapt to the robot, cannot capture how a proactive robot's actions shift both the environment and the user's behavior, thus rewarding policies that look strong against a static human but degrade once real users adapt. 
To address this, we contribute a closed-loop evaluation setup with a human model that leverages an LLM to adapt to robot actions, and benchmark existing methods to expose their shortcomings.

Finally, we present a model that operationalizes our framework for unprompted proactivity, learning from passive observation to anticipate user goals and adapt its actions. On our closed-loop setup, it substantially outperforms prior approaches developed under static evaluation assumptions.

In summary, we make the following contributions:
\begin{enumerate}
    \item A general formalism for effective proactive robots, a breakdown of the levels of proactivity, and a decomposition of unprompted proactivity, the highest level in our categorization
    \item An interactive closed-loop evaluation setup that uses an LLM as a proxy for human adaptation, and a characterization of how prior work's static, open-loop evaluation is misaligned with it
    \item A model that operationalizes our proposed unprompted-proactivity framework
\end{enumerate}

\section{A Formalism for Proactivity in Robotics}

Research on robot proactivity lacks a unified problem formulation or shared vocabulary. The terms \textit{proactivity} and \textit{anticipation} have been used loosely for many forms of agent initiative, including deciding the precise timing of an action~\cite{mascaro2023hoiabot, huang2016anticipatory, nemlekar2023transfer}, choosing to initiate an interaction~\cite{garrell2017teaching}, providing extra information~\cite{garrell2013proactive, peng2019design}, actively asking questions~\cite{dogan2025model, park2023clara, patel2025adapt}, and executing actions that serve anticipated needs~\cite{patel_stot_2023, patel_slatepro_2023}. Conversely, some works that enable proactivity do not use this terminology at all~\cite{puig_nopa_2023}. This inconsistency makes it challenging to connect prior works and build on them. 
Prior work~\cite{kochvandenbroek2024proactive} has noted the term's inconsistent use, but their treatment focuses on cataloguing the proactive interactions.
To address the resulting gap, we formalize proactive robot behavior, categorize it into three levels, and  characterize the technical sub-problems and computational approaches at the highest level of proactivity.


\subsection{Problem Formulation}
\label{sec:problem_formulation}

We consider an interactive setting in which a user and a robot share an environment and act over time, as illustrated in Fig.\ref{fig:factor_graph}. The environment 
state $s_t$
evolves under the joint actions of the human $a^h_t$ and the robot $a^r_t$, each of which can be no-op if the agent is idle. The interaction unfolds as the joint state-action sequence
\begin{equation}
\xi_{0:T} = s_0, a^r_0, a^h_0,\, s_1, a^r_1, a^h_1,\, \ldots, s_T, a^r_T, a^h_T
\end{equation}

We assume that the \textbf{human} performs actions according to a goal-oriented policy $\pi^h$ that pursues goals $g_t$ (Fig.\ref{fig:factor_graph}) such as setting the table for dinner or watering the plants, because completing them generates value. The value $v(g_t, s_t, h_t)$ of a goal depends on the goal itself, the state $s_t$, and the hidden state $h_t$, which captures internal states such as hunger and fatigue along with external influences such as other people and events. For example, a prepared meal is worth more when the user is hungry, and watering plants is worth more when the soil is dry. The total value of action sequence $\xi$ is the sum over all completed goals $\mathcal{G}_c = \{ g_t \}$,
\begin{equation}
V(\xi) = \sum_{g_t \in \mathcal{G}_c} v(g_t, s_t, h_t)
\end{equation}

The human takes actions $a^h_t = \pi^h(s_t, g_t)$ to reach their goals, and each action imposes a cost $c_h(a^h_t)$. The cost of a human action can be modeled with respect to various factors~\cite{malik2019complexity, smith2020assessing}, including physical effort, consumed resources, preference and enjoyment, and modulation by internal states~\cite{peternel_robot_2018} such as boredom or fatigue. The total cost associated with the human action sequence is
\begin{equation}
C_h(\xi) ~=~ \sum_t c_h(a^h_t)
\end{equation}
The human acts to maximize the \emph{net value} $V(\xi) - C_h(\xi)$, generating as much value as possible at the lowest cost.

The \textbf{robot} uses a policy $\pi^r$ to choose its actions, conditioned in general on a goal specification $g_t$, the human policy $\pi^h$, and the observed history $\xi_{0:t-1}$\footnote{Fig.\ref{fig:factor_graph} shows $\pi^r$ as Markovian for simplicity, but it can condition on the full history $\xi_{0:t-1}$ and other variables, including $\pi^h$ and $g_t$.},
\begin{equation}
a^r_t = \pi^r(g_t, \pi^h, s_t, \xi_{0:t-1})
\label{eq:robot_policy}
\end{equation}

Robot actions incur a cost $c_r(a^r_t)$ from physical wear and charge consumption, though this is generally small relative to the human effort a robot action can displace and is often ignored entirely in robotics applications. We assume the robot derives no inherent value from task completion, so its aim is to maximize the combined net value $V(\xi) - C_h(\xi) - C_r(\xi)$, with $C_r$ often assumed to be zero.

\begin{figure}
    \centering
    \includegraphics[width=1.0\linewidth]{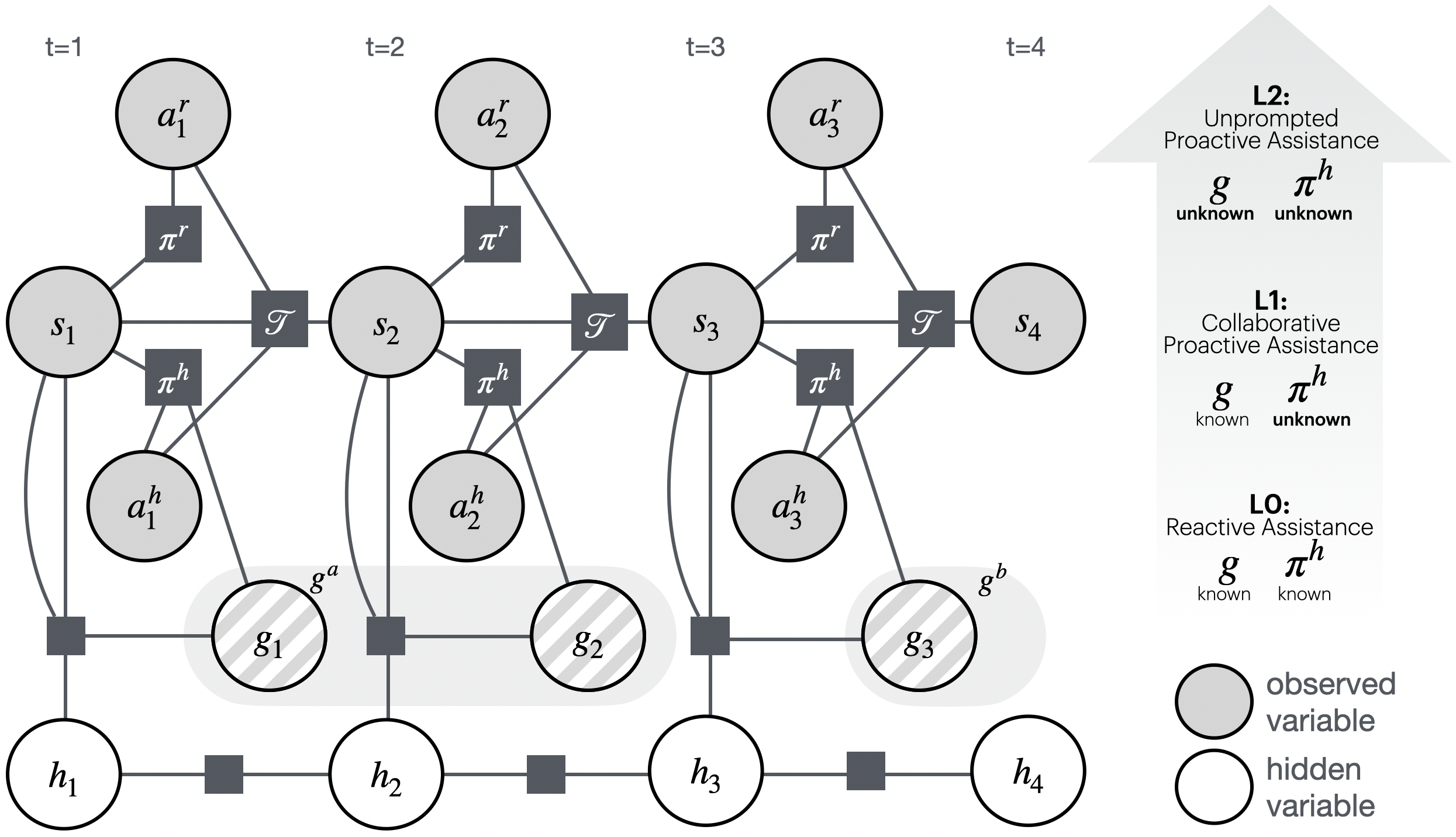}
    \setlength{\abovecaptionskip}{-8pt}
    \setlength{\belowcaptionskip}{-18pt}
    \caption{\small{Variables and processes in the proactive assistance problem. The fully observed environment evolves under joint human $a^h_t$ and robot $a^r_t$ actions, while the hidden user state $h_t$ evolves in parallel and shapes the active goal $g_t$ at each timestep. We introduce three levels of proactivity for the robot policy with increasing unknowns.}}
    \label{fig:factor_graph}
\end{figure}


We define the counterfactual trajectory $\xi'$ as the sequence that would unfold in the robot's absence, with the human acting alone under the same dynamics. The robot's actions change both the cost of the human's actions toward a goal and which goals are achieved. Because these effects propagate through the trajectory, we \textbf{measure overall effectiveness of robot actions} by comparing $\xi$ against $\xi'$. This comparison yields a change in value and a change in cost,
\begin{equation}
\Delta V(\pi^r) = V(\xi) - V(\xi')
\end{equation}
\begin{equation}
\Delta C(\pi^r) = \big( C_h(\xi) + C_r(\xi) \big) - \big( C_h(\xi') + C_r(\xi') \big)
\end{equation}

The net $\Delta V - \Delta C$ characterizes the value of the robot's presence. The robot can be \textbf{effective} by reducing the cost of achieving a goal the user would have completed anyway ($\downarrow\Delta C$), achieving a goal that would otherwise go uncompleted ($\uparrow\Delta V$), or freeing the user's time for additional valued goals that would otherwise be missed ($\uparrow\Delta V$). Conversely, it is \textbf{ineffective} when an incorrect action forces the human to spend extra effort undoing or redoing actions ($\uparrow\Delta C$).

\subsection{Levels of Proactive Assistance}

A robot must act proactively when it is not explicitly given all the information needed to determine its actions. Formally, the robot policy $\pi^r$, defined in Eq.~\ref{eq:robot_policy}, is conditioned on the goal $g_t$, the human policy $\pi^h$, the interaction history $\xi_{0:t-1}$, and the current state $s_t$, but some of this information may be unavailable at run time. For example, the robot may not be told that the user wants the home tidied, or it may know the goal but not the user's preferences or likely next actions. The amount of missing information determines the level of proactivity required. We define three levels with increasing uncertainty, shown in Fig.~\ref{fig:factor_graph}: a reactive robot, where all variables are known; proactive collaboration, where the goal is known but the human policy is not; and unprompted proactive assistance, where no goal is provided.

\textbf{L0: Reactive Assistance $\pi^r(g_t, \pi^h, s_t, \xi_{0:t-1})$.} 
We define Reactive Assistance as the condition where no proactivity is required, as the robot has access to both $g_t$ and $\pi^h$.
Knowing the user's goal $g_t$ lets the robot plan toward it, and knowing the low-level policy $\pi^h$ lets it predict effective assistive actions with the human in the loop, or treat the human as a null policy when acting alone. 
Such exact knowledge is available only in rigid settings where the goal is fully specified and either the user is absent, or the order of user actions is predefined, such as in industrial manufacturing.

\textbf{L1: Collaborative Proactive Assistance $\pi^r(g_t, s_t, \xi_{0:t-1})$.} 
We define Collaborative Proactive Assistance as the second level of proactivity, where the robot has access to $g_t$, but not $\pi^h$.
Lacking knowledge of the human policy, the robot cannot know the user's next action and must infer their preferred strategy from the history $\xi_{0:t-1}$. Prior works learn preferred strategies from collaborative data~\cite{zhao2022coordination} or passive observation~\cite{huang2016anticipatory,nemlekar2023transfer}, or learn the timing of actions for fluid interaction~\cite{mascaro2023hoiabot,koppula2015anticipating}. Newer work~\cite{huang2025hierarchical} extends this idea to include mode switching that returns brief control to the human for failure recovery. Still, exact goal knowledge is realistic only in industrial settings with concrete, repetitive tasks such as multi-step assembly.

\textbf{L2: Unprompted Proactive Assistance $\pi^r(s_t, \xi_{0:t-1})$.} 
We define Unprompted Proactive Assistance as the highest level of proactivity, where the robot does not have access to either $g_t$ or $\pi^h$.
Since no goal is specified, the robot must act as a proactive partner that fulfills needs without being asked. It predicts effective actions without knowing $g_t$, relying instead on past interactions $\xi_{0:t-1}$ to infer the user's needs. 
Existing methods, detailed in Sec.~\ref{sec:existing_approaches}, approach this either by inferring the user's current goal from their actions and planning toward it~\cite{puig_nopa_2023, puig2020watch, zhu2026proact}, or by predicting the user's future actions directly from routine patterns learned from their observed history~\cite{patel_stot_2023, patel_slatepro_2023, bartoli_streak_2025}.
Unlike the previous levels, which support only a single fixed goal, this level lets the robot act toward any future goal $g_{t:\infty}$.

Among these levels, only \textit{unprompted proactive assistance} fully relieves the user of the cognitive burden of delegating tasks, and is the focus of the rest of this paper.




\begin{figure*}[t]
  \centering
  \begin{minipage}[t]{0.66\linewidth}
  \vspace{0pt}
    \centering
    \captionof{table}{\small{Open problems in unprompted proactive assistance, grouped into grounding, anticipation, and planning, with inputs, outputs, and prior works.}}
    \label{tab:open-problems}
    \vspace{2pt}
    \scriptsize
    \setlength{\tabcolsep}{2pt}
    {
    \renewcommand{\arraystretch}{1.1}
    \begin{tabular}{@{}c 
                         >{\centering\arraybackslash}p{0.5cm}
                         >{\raggedright\arraybackslash}p{3.0cm}
                         >{\centering\arraybackslash}p{2.2cm}
                         >{\centering\arraybackslash}p{1.3cm}
                         >{\raggedright\arraybackslash}p{3.8cm}@{}}
  \toprule
   &  &  \textbf{Problem} & \textbf{Input} & \textbf{Output}
   & \textbf{Prior Works} \\
  \midrule
  \multirow{2}{*}{\rotatebox[origin=c]{90}{\tiny\textbf{Grounding}}}
  & G1 & Goal-space representation
     & $\xi_{0:t}$
     & $\mathcal{G}$
     & - \\
   \cmidrule(l){2-6}
  & G2 & Action-goal association
     & $\xi_{0:t},\ \mathcal{G}$
     & $g_{0:t}$
     & \cite{Nakahashi2015ModelingHU,ZhiXuan2020OnlineBG,puig2020watch,Baker2009ActionUAA} \\
  \midrule
  \multirow{3}{*}{\rotatebox[origin=c]{90}{\tiny\textbf{Anticipation}}}
  & A1 & Human policy learning
     & $\xi_{0:t},\ \mathcal{G},\ g_{0:t}$
     & $\pi^h$
     & \cite{Nikolaidis2014EfficientML,Jayanthi_strategy_2022} \\
   \cmidrule(l){2-6}
  & A2 & Goal inference
     & $\xi_{t-\Delta t:t},\ \mathcal{G},\ \pi^h$
     & $g_t$
     & \cite{Ramrez2010ProbabilisticPR,puig_nopa_2023,Lynch2019LearningLPA,zhu2026proact} \\
   \cmidrule(l){2-6}
  & A3 & Goal anticipation
     & $\xi_{0:t},\ \mathcal{G},\ g_{0:t}$
     & $g_{t:t+\Delta t}$ or $\phi$
     & \cite{patel_stot_2023,patel_slatepro_2023,bartoli_streak_2025} \\
  \midrule
  \multirow{3}{*}{\rotatebox[origin=c]{90}{\hspace{2mm} \tiny\textbf{Planning}}}
  & P1 & Cost modeling
     & $s_t,\ a^r_t,\ a^h_t$
     & $c_h(a^h_t)$ 
     &
     \cite{smith2020assessing, patel2025taaco,dogan_grace_2025,banerjee2018robot} \\
   \cmidrule(l){2-6}
  & P2 & Coordinated action selection
     & $\xi_{0:t},\ c_h(.),\ \pi^h,$ $g_{t:t+\Delta t}$
     & $a_t^r$
     & \cite{hadfield2016cooperative} \\
  \bottomrule
  \end{tabular}
  }
  \end{minipage}\hfill
  \begin{minipage}[t]{0.335\linewidth}
  \vspace{0pt}
    \centering
    \includegraphics[height=1.7in]{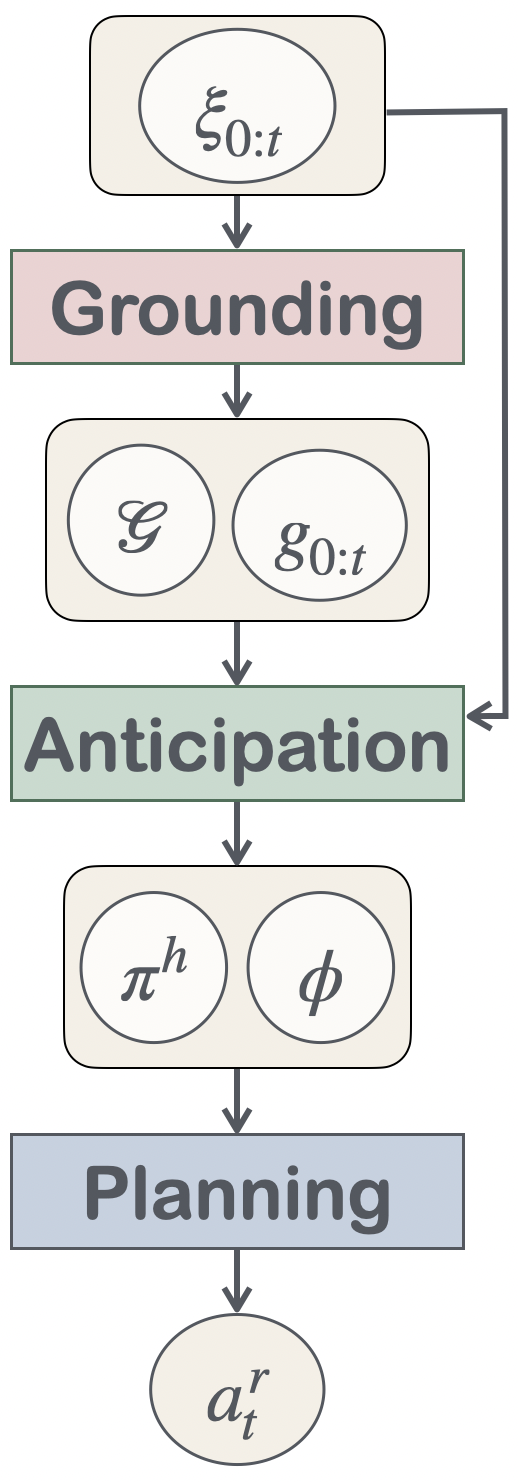}
    \vspace{-2mm}
    \captionof{figure}{\small{Dependency structure of sub-problems and how they build on each other}}
    \label{fig:subproblems}
  \end{minipage}
\vspace{-5mm}
\end{figure*}

\section{Algorithmic Framework for L2 Proactivity}
\label{sec:sub_problems_framework}

In L2, unprompted proactive assistance, the policy takes the form $a^r_t = \pi^r(s_t, \xi_{0:t-1})$, conditioning each robot action on the current state and the full history. To address this, we propose a framework composed of sub-problems, and show how existing works address a scoped version of the problem.

\subsection{Sub-Problems in Unprompted Proactive Assistance}

We decompose Unprompted Proactive Assistance into three high-level steps (Fig.\ref{fig:subproblems}): \textit{grounding} models the intentional structure behind user actions, \textit{anticipation} reasons over temporal  user behaviors to provide an anticipated need, goal, or action, and \textit{planning} predicts robot actions by optimizing towards an anticipated variable. 
Each of these challenges comprises several sub-problems, shown in Table~\ref{tab:open-problems}, which facilitate unprompted proactivity\footnote{Solving every sub-problem is not a prerequisite for effective proactive behavior, as the existing approaches in Sec.~\ref{sec:existing_approaches} show.}. 
No single method solves unprompted proactive assistance end to end. We discuss works that are most relevant  here and provide a more detailed literature review in Appendix~\ref{app:lit}.

\subsubsection{Grounding}
The robot obtains a flat sequence of user actions and environment states $\xi_{0:T}$, from which it must infer the intentional structure underlying user behavior by jointly identifying the set of user goals or tasks and associating each observed action with its corresponding goal.
The robot must first construct a \textbf{goal-space representation (G1)} that curates the set of goals $\mathcal{G}$ over which it reasons.
Since a proactive robot defines its own goals rather than receiving them, the notion of a \textit{goal} is under-defined and may span abstraction levels (pick apple vs. make breakfast). The chosen space should be hierarchical or coarse enough to model intent yet fine enough to separate variations. No prior work curates a goal space autonomously.
Given a goal space, \textbf{action-goal association (G2)} assigns observed actions to goals, yielding $g_{0:t}$. 
The core challenge is disentangling interleaved actions across concurrent goals and isolating extraneous ones. Prior work infers this structure with symbolic~\cite{puig2020watch,Baker2009ActionUAA} and learned~\cite{Nakahashi2015ModelingHU,ZhiXuan2020OnlineBG} models.

\subsubsection{Anticipation}
Given the past behavior observations $\xi_{0:t}$ and their corresponding goals $g_{0:t}$, the anticipation objective is to learn the function governing the user's current and future goals $g_{t+\Delta t}$ and the policy $\pi^h$ by which they act towards a given goal. 
This requires \textbf{human policy learning (A1)} to learn the user's goal-oriented policy $\pi^h$ from observed goal-conditioned actions. 
The core challenge is 
adapting based on limited user-specific data. Prior works address data sparsity by clustering users or fitting mixtures~\cite{Nikolaidis2014EfficientML,Jayanthi_strategy_2022}, but rely on explicit demonstrations.
To determine what the user is currently trying to accomplish, \textbf{goal inference (A2)} infers the user's current goal $g_t$ online from recent actions.
The challenge is disambiguating among many goals and providing assistance despite uncertainty over goals. Prior works form a belief over goals with actions as likelihoods, addressed through 
symbolic~\cite{puig_nopa_2023,Ramrez2010ProbabilisticPR} 
learned~\cite{zhu2026proact}
, and hybrid~\cite{Lynch2019LearningLPA} 
methods.
Looking beyond the current goal, \textbf{goal anticipation (A3)} predicts future goals $g_{t:t+\Delta t}$ that the user has not yet begun, via the temporal generation process $\phi$. 
While the hidden state $h_t$ driving the temporal goal generation is inaccessible to the robot, it can marginalize over $h_t$ and model the resulting goal generation process as $g_t = \phi(g_{0:t-1}, s_{0:t-1})$, a function of state and goal history. The core challenge is learning from wider context, such as habitual behavior patterns. Closest works model state evolution as a temporal sequence over scene graphs with Graph Neural Networks~\cite{patel_stot_2023, patel_slatepro_2023, bartoli_streak_2025}. 

\subsubsection{Planning}

Given the ability to predict future rollouts from $\phi: g_{0:t-1},s_{0:t-1} \rightarrow g_t$ and $\pi^h: s_t, g_t \rightarrow a^h_t$, we model planning as $\pi^r = \arg\max_{\pi^r} \mathbb{E}_{{\xi}^{\pi^r}}[\Delta V ({\xi}^{\pi^r}) - \Delta C ({\xi}^{\pi^r})]$, maximizing expected net value change over the anticipated future, with $\Delta V$ and $\Delta C$ as defined in Sec.~\ref{sec:problem_formulation}.
Evaluating candidate plans requires \textbf{cost modeling (P1)} to learn the user's cost model $c_h(a^h_t)$, spanning the physical effort of an action, the user's preference for doing it, the overhead of interacting with the robot, and/or the disturbance of getting in the way.
The challenge is that these components are personal, context-dependent, and depend on the robot's own concurrent actions, making them hard to quantify. Prior work models effort for task allocation~\cite{smith2020assessing},
learns assistive preferences~\cite{patel2025taaco,dogan_grace_2025}, 
and estimates interruptibility~\cite{banerjee2018robot}, but each captures only part of the cost.
Finally, \textbf{coordinated action selection (P2)} creates a policy to predict robot action $a^r_t$ that minimizes user cost, based on the learned cost $c_h$, human policy $\pi^h$, and anticipated goals $g_{t:t+\Delta t}$. The core challenge is optimizing against a learned, adaptive human rather than a ground-truth or scripted one. Moreover, the coordination granularity could vary, requiring finer coordination when acting on the current goal and coarser when preparing for a future one. Closest works use cooperative inverse RL~\cite{hadfield2016cooperative}
, but are limited to simulated domains with known rewards.

\smallskip
Several themes span the above sub-problems, including joint optimization across sub-problems to improve over solving each alone, personalization to learn user-specific patterns, continual learning to adapt to drifting needs without forgetting, and user interaction to resolve uncertainty and answer user queries.

\subsection{Existing Integrated Solutions for L2 Proactivity}
\label{sec:existing_approaches}

Two primary families of vertically integrated approaches have emerged for unprompted assistance, but each addresses only a portion of the sub-problems outlined in Sec.~\ref{sec:sub_problems_framework}.

The first family~\cite{puig_nopa_2023, puig2020watch, zhu2026proact} performs online \textbf{goal inference} followed by goal-conditioned planning, addressing the goal-inference sub-problem while assuming the goal space, action-goal association, and human policy are given rather than learned, and avoiding prediction of future goals entirely. This scoping leads to three limitations that our work addresses. First, these methods require the full set of possible goals along with ground-truth goal annotations for past actions, which would demand intensive labeling or user feedback. Second, they assist only with the current task, and only once the belief concentrates, which in a large goal space can take too long to leave any time to act. 
Third, they ignore habits and preferred variants, since under a universal goal distribution an 8am coffee cannot be predicted from the time of day, and the preferred cream-and-sugar variant cannot be anticipated before the user adds it. The next family addresses this by modeling habits and variation.

The second family leverages observed history for direct \textbf{temporal user action prediction}~\cite{patel_stot_2023, patel_slatepro_2023, bartoli_streak_2025, saavedra2026perceptua} and imitates predicted actions, bypassing explicit goal inference. These methods learn what the user tends to do next from temporal patterns, and at deployment the predicted action becomes the robot's action. In our formalism, this can be interpreted as selecting a goal space at the finest granularity such that it coincides with the action space, or equivalently as a marginalization over the goal $g$ in addition to the hidden context $h$ to directly learn an action predictor $a^h_t = \phi(a^h_{0:t-1}, s_{0:t-1})$. These approaches therefore skip explicit goal inference and human-policy learning. They predict minutes to hours ahead, eliminating the need for close collaboration, and directly act on the predictions, implicitly assuming the cost function to be uniform across actions that the human performs. The core limitation of this family of approaches is complementary to that of goal inference, since operating at the level of actions discards the intentional structure of behavior that a goal-level model can exploit.
\section{L2 Evaluation Setup}
\label{sec:eval_setup}

Developing and evaluating unprompted proactivity requires an interactive setting with a realistic, adaptable human $\pi^h(s_t, g_t)$ and a metric for robot effectiveness $\Delta V - \Delta C$. Because studying proactive behavior with real humans is costly and risky, especially given the need for long-term interaction, prior work has relied on static datasets~\cite{puig2020watch, patel_stot_2023}. 
By design, proactive robot actions affect the environment state $s_t$, altering the user's future actions $\pi^h(s_t, g_t)$ and value of different goals $v(g_t, s_t, h_t)$, and resulting in cascading effects on the robot's own future actions $\pi^r(s_t, \xi_{0:t})$.
Evaluating on static data cannot capture these intervention-induced changes. It evaluates the robot against a fixed sequence induced by human behavior alone $\xi^h=s_0,a^h_0,s_1,a^h_1,...$, thereby missing cascading effects of robot actions on both the user and the robot and potentially overstating real-world performance. 

Specifically, offline evaluation scores a robot policy by replaying a fixed record of human behavior and evaluating each robot action against that static trajectory, resting on two assumptions: 1) that the state distribution the robot sees stays the same as training, and 2) that the human's goals are fixed. As the human and robot co-adapt, both assumptions fail, and the \textbf{policy that appears optimal under those assumptions no longer is}.
First, when a robot action does not match what the user would have done, the state drifts and the robot errs more often, degrading to erring at every step in finite time, a compounding cost offline evaluation never incurs since it scores actions only over dataset states. Second, effective assistance changes what the user does next, since a value-maximizing user freed from one task reallocates to a higher-value goal, so the offline-optimal policy is strictly suboptimal against the goals the user actually adopts. We derive both gaps formally in Appendix~\ref{app:misalignment}. In this section, we present our closed-loop evaluation setup and empirically show that existing methods lose substantial performance relative to the offline setting.

\subsection{Closed Loop Evaluation Setup}
\label{sec:llm_eval}

Closed-loop evaluation, unlike an offline setup, captures both action-induced distribution shift and the human's adaptation to the robot (purple arrows in Fig.~\ref{fig:system}). We contribute the first such closed-loop setup, modeling human adaptation with an LLM and metrics capturing the effect of robot actions on user effort. To create this setup, we adopt the action space of object relocation similar to prior works~\cite{puig_nopa_2023, patel_slatepro_2023}.


\subsubsection{Human Model}

We build a closed-loop simulation in VirtualHome~\cite{puig2018virtualhome}, driving user behavior with the HOMER+ dataset~\cite{patel_slatepro_2023} and modeling the user's adaptation to the robot with an LLM (gpt-4o-mini). HOMER+ provides realistic behaviors across three simulated households over several weeks $\xi^h = s_0, a^h_0, ...$, sampled from crowdsourced routine patterns and executable in VirtualHome. The simulated user agent executes the action sequences from HOMER+, skipping actions the robot already performed, relocating an object from its robot-induced location when it is not where expected, and otherwise executing each action unchanged. 

To enable adaptation to new states produced by robot actions, we create two user adaptation models $\pi^h(s_t,g_t)$. The first one uses an LLM as a proxy to determine the change in user actions beyond $\xi^h$. For instance, if the robot moved a cereal bowl to the counter, and the user planned on putting it on the dining table, the LLM can choose to leave it there and move it later or just make cereal on the counter. We give the LLM examples of how to judge robot actions and the day's context, which includes a timestamped list of past user and robot actions $\xi_{0:t}$, with each user action annotated as task-related or corrective, and future user actions based on HOMER+ $\xi^h_{t:t+\Delta t}$, over a fixed predictive window $\Delta t$, set to 2 hours (full prompt in Appendix~\ref{app:prompts}). The LLM is queried for all robot actions $a^r_t$ except when the robot action directly matches the next user action for that object within the predictive window. The LLM assigns the action to one of three responses: accept it, undo it by returning the object, or redirect it to where the user actually needs it.
The second is a scripted user model, which does not explicitly respond to the robot. The scripted user model would pick up the cereal bowl from where the robot left it, and bring it to the dining table at the usual time. While the LLM-based model responds in a more natural manner, it also introduces more stochasticity compared to the scripted model.


\subsubsection{Metrics}
To measure value generated by the robot $\Delta V - \Delta C$, we compare against the actions from HOMER+, which the user would have taken in the robot's absence as the counterfactual $\xi' = \xi^h$. We categorize human actions in the baseline sequence into: \textit{i)} \textit{Skipped} action, indicating user effort saved by a correct robot action ($\Delta C = - c_h(a^h)$), \textit{ii)} \textit{Done as-is} action, unaffected by the robot ($\Delta C = 0$), and \textit{iii)} \textit{Corrected} action, indicating a perturbed action $\tilde{a}^h$ due to a wrong robot action ($\Delta C = c_h(\tilde{a}^h) - c_h(a^h)$), requiring the user to move the object from where the robot placed it. Separately, a wrong robot action can also cause the user to \textit{undo} it through an additional action ($\Delta C = c_h(a^h)$), adding a new user action absent from the baseline. In both human models the goal distribution stays fixed, apart from the LLM user occasionally acting on a future goal sooner, hence $\Delta V \approx 0$. Quantifying each action's cost exactly would require additional information, so we assume a uniform cost per action, except for perturbations, where recovering from an unexpected robot action carries an added cost (e.g. searching for a cereal bowl the robot shelved in the wrong place). 

Under this model, the robot increases net value by reducing the number of net user actions, by $|a^h| * c_h(a^h)$, but it can also reduce value by perturbing user actions, by $|\tilde{a}^h| * (c_h(\tilde{a}^h) - c_h(a^h))$. Because the relative cost of a perturbation is task- and user-dependent, instead of combining these into a single number, we report both components separately: \textbf{net actions saved}, as the difference between \textit{skipped} actions and \textit{undo} actions, and \textbf{perturbed actions}, as the \textit{corrected} actions.
We report both metrics as a percentage of baseline user actions. Finally,
we report the \textbf{F1 score} to enable comparison with prior evaluations, using skipped as true positives, undo and corrected as false positives, and done as-is as false negatives.
 

\begin{figure*}[t]
    \centering
    \includegraphics[width=0.78\linewidth]{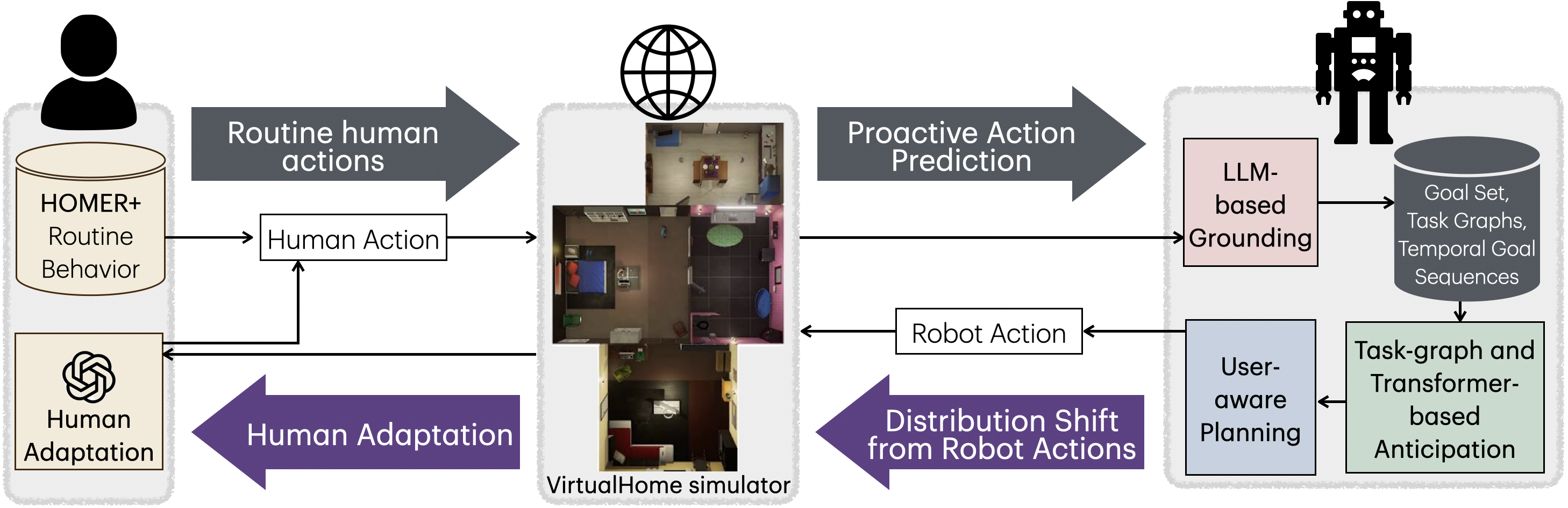}
    \setlength{\abovecaptionskip}{2pt}
    \setlength{\belowcaptionskip}{-16pt}
    \caption{\small{Our proactivity method, Grounding, Anticipation, and Planning (GAP), in a closed loop evaluation setup. Closing the evaluation loop accounts for effects (purple arrows), in addition to those that open-loop does (grey arrows)}}
    \label{fig:system}
\end{figure*}

\subsection{Empirical Investigation of Evaluation Misalignment}

We evaluate prior methods, SLaTe-PRO~\cite{patel_slatepro_2023} and STREAK~\cite{bartoli_streak_2025} on our closed-loop setup and compare against the offline evaluation on static datasets used in the published works. In the closed-loop setup, we execute each model's robot actions to advance the state, and retrain each model every five days on all past data, now containing the joint user and robot action sequences. Table~\ref{tab:offline_vs_closedloop} compares F1 scores between the offline and closed loop settings, and shows that offline evaluation drastically overstates performance for both methods. Both score above 0.6 offline on the HOMER+ dataset, but collapse in the closed-loop setup, especially under the more stochastic LLM human.

\begin{table}[h]
\centering
\caption{\small{Offline versus closed-loop F1 scores for existing methods, illustrating the misalignment between the evaluation methods.}}
\label{tab:offline_vs_closedloop}
\renewcommand{\arraystretch}{1.0}
\begin{tabular}{lccc}
\toprule
 & & \multicolumn{2}{c}{Closed-Loop} \\
\cmidrule(lr){3-4}
Method & Offline & LLM User & Scripted User \\
\midrule
SLaTe-PRO & 0.634 & 0.216 & 0.264 \\
STREAK    & 0.677 &   0.002    &  0.296     \\
\bottomrule
\end{tabular}
\vspace{-6mm}
\end{table}

\section{Adaptive Unprompted Proactivity Model}
\label{sec:method}



\label{sec:model_design}

We contribute a proactive robot model, GAP, by instantiating the Grounding, Anticipation, and Planning sub-problems of Sec.~\ref{sec:sub_problems_framework} as a single pipeline that learns from past observations (Fig.\ref{fig:system}). A language model curates a goal space from raw activity logs (G1, G2), per-goal task graphs and a particle filter learn and infer the user's goals online (A1, A2), a transformer predicts upcoming goals (A3), and a planner turns these anticipations into candidate robot actions (P2) and selects among them by task-graph preconditions, timing, and user acceptance (P1).

We use an LLM (gpt-4o-mini) to curate an open-vocabulary set of goals (G1). The curator maintains a vocabulary of goals, with each entry containing a name (e.g.\ \texttt{cereal\_or\_oatmeal\_breakfast}), a natural-language description (e.g.\ \textit{preparing and eating a cereal or oatmeal breakfast}), the first day the goal was observed, and its observation count. As observations arrive, the curator can create new goals and modify or merge existing ones.
Alongside curation, the LLM tags each day's past action observations with a goal from the vocabulary (G2). For each day it generates goal intervals with start and end times, assigning the scene-graph transitions in each interval to the labeled goal. A deterministic repair step sanitizes the proposed segments by dropping those with invalid times, either outside the bounds of the action sequence or of negative duration. To avoid forcing arbitrary assignments, the LLM may label a small fraction of ambiguous actions unassigned.

For each observed goal, we learn a directed task graph capturing how the user tends to perform it (A1) and at deployment, we use a particle filter over task graphs to perform online goal inference (A2). Each node is a movement of an object to a destination, recording its visit count and average timing offset from the goal's start. Edge weights are transition probabilities, estimated from observed node-to-node transitions and normalized over each node's outgoing edges, with Laplace smoothing to avoid extreme values where evidence is sparse.
Each particle maintains a set of open goal hypotheses, the last completed step of each, and the day's goal history. Every user action results in a likelihood of that action aggregated over three possibilities: (1) continuing an open goal, scored by the highest-probability path from the last completed step to the observed node; (2) opening a new goal, scored by that goal's time-based prior and the observed node's path from the start node; and (3) an \textit{unclear} action, with fixed probability $\epsilon = 0.001$. We advance particle state by sampling one possibility per particle for user actions, and stochastically marking the corresponding node as satisfied for robot actions. If the user exactly undoes a robot placement, particles that had marked that node as robot-satisfied are down-weighted by $\epsilon/(\epsilon+p_{node})$, where $p_{node}$ is the marginal probability of that node in the task graph. 
To maintain particle health, we apply a temporal decay to each open goal to account for unmatched observations, close a goal when it reaches its end node or goes stale from missing observations, and resample whenever the effective sample size drops below half the particles.

While the particle filter recognizes the user's current goal, to anticipate goals that have not yet begun, we train a causal transformer on the sequence of curated goal segments (A3). We represent each goal segment by a learned projection of the LLM embedding of the goal name together with its start, end, and duration, with start and end times encoded by sinusoidal functions as in the baselines~\cite{patel_slatepro_2023, bartoli_streak_2025}. Two heads process the last-timestep output to predict the next goal and its expected start time. The first scores each goal by a linear projection followed by dot-product similarity against the goal embeddings, and the second predicts a goal's start time by concatenating the output with each goal embedding, projecting, and comparing against sinusoidal time embeddings at ten-minute intervals. Hyperparameters and fixed constants are listed in Appendix~\ref{app:implementation}.

\begin{table*}[t]
\centering
\setlength{\abovecaptionskip}{3pt}
\caption{\small{Closed-loop performance under the LLM and scripted user, with standard deviation intervals between households. 
}}

\label{tab:results}
\setlength{\aboverulesep}{0pt}
\setlength{\belowrulesep}{0pt}
\renewcommand{\arraystretch}{1.0}
\begin{tabular}{l|rrr|rrr}
\toprule
   & \multicolumn{3}{c|}{LLM User} & \multicolumn{3}{c}{Scripted User} \\
  \cmidrule(lr){2-4} \cmidrule(lr){5-7}
  Method & Net Saved (\%) $\uparrow$ & Perturbed (\%) $\downarrow$ & F1 $\uparrow$ & Net Saved (\%) $\uparrow$ & Perturbed (\%) $\downarrow$ & F1
  $\uparrow$ \\
  \midrule
  SLaTe-PRO          & 2.1 $\pm$ 7.5 & 18.1 $\pm$ 3.1 & 0.216 $\pm$ 0.034 & 16.6 $\pm$ 3.0 & 11.8 $\pm$ 1.4 & 0.264 $\pm$ 0.043 \\
  STREAK             & -159.9 $\pm$ 50.3 & \textbf{10.6 $\pm$ 10.8} & 0.002 $\pm$ 0.001 & 20.3 $\pm$ 7.3 & 48.3 $\pm$ 19.0 & 0.296 $\pm$ 0.105 \\
  Plain LLM          & -174.5 $\pm$ 26.5 & 62.1 $\pm$ 4.5 & 0.076 $\pm$ 0.022 & 19.5 $\pm$ 4.8 & 25.2 $\pm$ 0.8 & 0.310 $\pm$ 0.064 \\
  \ourMethod\ (ours) & \textbf{17.7 $\pm$ 1.1} & 11.7 $\pm$ 1.9 & \textbf{0.388 $\pm$ 0.007} & \textbf{30.8 $\pm$ 1.8} & \textbf{10.1 $\pm$ 0.8} &
  \textbf{0.452 $\pm$ 0.024} \\
  \midrule
  Oracle             & 44.4 $\pm$ 1.6 & 8.2 $\pm$ 1.1 & 0.617 $\pm$ 0.015 & 51.3 $\pm$ 0.0 & 8.8 $\pm$ 1.9 & 0.651 $\pm$ 0.006 \\
  \bottomrule
\end{tabular}
\vspace{-6mm}
\end{table*}

The planner proposes candidate object placements from the inferred current goal and the predicted future goals (P2). For the current goal, it anchors the task graph to absolute time using the last observed user action, then predicts upcoming actions with their expected times and confidence probabilities. For an upcoming goal, it uses the actions and confidence predicted by the transformer. From both sources, the planner proposes the candidate actions whose confidence exceeds a tunable threshold  $c_{\mathrm{thresh}}$.
We score each candidate action by a heuristic utility that captures the timing, preference, and sequential structure of actions (P1), which is combined with the action confidence and measures helpfulness or cost reduction. The first condition checks that the action's path over the task graph contains no unsatisfied precondition actions on that object, since acting before a precondition is met yields no benefit. The second condition requires the action's expected time of occurrence to fall within a per-action lead-time threshold, learned by starting from an initial value and adjusting over time based on successful and reverted placements. The third factor scales the utility by a learned per-object-destination acceptance estimate, estimated from how often the user builds upon the robot's placement rather than undoing it, with a short within-day backoff that skips a placement the user just rejected.

To track the shifting distribution, we recency-weight the task-graph counts for node observations, robot-actuated movements, and transitions by a parameter $\gamma$, fitted during each retraining to best explain the newest observations. We recency-bias the transformer's training data using the same $\gamma$. We retrain all models every $n_{retrain}$ days, set to five in all our experiments, adapting on all past data every five days.

\section{Results}

\label{sec:eval}






We evaluate our model \ourMethod\ (Sec.~\ref{sec:method}) in the closed-loop setup (Sec.~\ref{sec:llm_eval}), reporting the net actions saved, perturbed user actions, and F1 score.

\textbf{Baselines:} We compare \ourMethod\ against four baselines: 
\begin{flushitemize}
\item \textit{SLaTe-PRO}~\cite{patel_slatepro_2023} autoregressively predicts the next state iteratively, and executes predicted user relocations over a predictive window above a confidence threshold.

 \item  \textit{STREAK}\footnote{To avoid divergent, unreportable behavior, we cap STREAK for LLM-user at two actions per timestep and forbid placing a node on its child (in the dataset) to prevent impossible user actions.}~\cite{bartoli_streak_2025} predicts the next state using a Markovian model and uses continual learning to regularize successive training rounds.

\item \textit{Plain LLM} prompts gpt-4o-mini at each timestep to predict the user's next placements, given up to the 14 most recent days of activity and the moves seen so far that day.

\item \textit{Oracle} is an upper bound that knows the ground-truth future relocations and acts exactly one tick ahead. It marks the ceiling on the assistance achievable, as the robot cannot preempt multiple moves that fall within the same timestep.
\end{flushitemize}

Table~\ref{tab:results} reports the mean and standard deviation across households for each metric. \textbf{\ourMethod\ outperforms all non-oracle baselines} across every metric under the closed-loop setup, with the sole exception of perturbed actions under the LLM user, where STREAK reports fewer. Net actions saved and perturbed actions measure two components of the robot's net value, so they must be read together. STREAK's low perturbed count comes with drastically worse net actions saved (a net -159.9\% under the LLM user), so its low perturbation does not reflect better assistance. The low performance of Plain LLM baseline shows that without the structured grounding, anticipation and planning framework, the commonsense priors in an LLM fail to translate into well-timed, personalized assistance. Thus, GAP's performance cannot be attributed to the LLM used in the grounding step alone. Finally, the oracle result helps contextualize the results of all methods by establishing the ceiling of assistance for a proactive robot in this setting, since the robot cannot assist with actions that happen in quick succession.

The \textbf{baselines SLaTe-PRO and STREAK score poorly because they are built for an open-loop setup}. When robot actions shift the distribution, they do not have a mechanism to cope with the changes, suffering from the misalignment described in Sec.~\ref{sec:eval_setup}. This effect is exacerbated in the LLM user, which explicitly adapts its actions to the robot, compared to the scripted user. GAP handles the closed loop effects of robot actions in two ways. 
First, it decomposes reasoning into goal- and action-level modeling. User adaptation within the same goal does not affect the goal-level model, and at the action level the task graphs are not sensitive to the robot performing an action, enabling GAP to adapt its remaining actions within the goal. Second, its heuristic planner learns directly from interaction over time, through the lead-time and acceptance conditions, and responds to the user by not repeating rejected actions.

The standard deviations reveal that \textbf{the baselines fail by destabilizing and compounding their own errors}, rather than by settling into a stable, mildly suboptimal policy. This form of failure is severe and erratic, with mean net actions saved falling to -159.9\% and -174.5\% for STREAK and Plain LLM respectively under the LLM user, with standard deviations as high as 50.3. \ourMethod, by contrast, remains stable and strongly positive (17.7\% $\pm$ 1.1 and 30.8\% $\pm$ 1.8), with the tightest intervals of any non-oracle method, indicating assistance that is consistent across households rather than averaging over runs that range from helpful to harmful.

Even though STREAK is built to handle distribution shift in user actions, we find that the \textbf{continual learning alone does not address the self-induced distribution shift} in the closed loop setup, which creates a feedback loop. One way to break this loop is to train only on user actions and the robot actions the user accepts, discarding rejected ones.  We find this drops performance, reducing net actions saved by 14.3\% points, and raising perturbed actions by 19.1\% points on the scripted user. We posit that this is because ignoring wrong robot actions moves the training distribution farther away from the real closed-loop distribution seen at test time.

\textbf{GAP balances reliable current goal-focused help against harder but valuable assistance towards future goals.} We separate \ourMethod's actions based on their two sources: preparation for a future goal (from A3) and assistance toward the current goal (from A2). We find that future goal assistance accounts for 32.8\% of robot actions with the scripted user and 25.3\% with the LLM user. Anticipating a future goal is the harder of the two, because the robot must first predict which goal will occur before predicting its actions. Accordingly, precision is lower for future-goal assistance than current-goal assistance (0.64 vs 0.80 scripted, 0.53 vs 0.62 LLM). GAP draws on both sources, leaning on current-goal actions for reliable help while also acting on the less certain future goals to expand the assistance it can offer.

\section{Conclusion and Limitations}
We formalize proactivity in robotics as three levels, and for the highest, unprompted proactivity, we identify a gap in how prior work is evaluated offline. We close this gap with an interactive closed-loop setup, decompose the problem into its component subproblems, and propose a method that instantiates them and outperforms prior approaches under closed-loop evaluation.

Our formalism maps the open problems of unprompted proactivity, and our setup and method address a subset of them. Our user-adaptation model is limited to the goals already present in HOMER+, because true goal adaptation requires a real-world, study-based understanding and modeling of how users adapt to robots over longitudinal time frames. Our user adaptation model relies on an LLM proxy, which has been supported by prior work~\cite{zhang2023large}, but real user studies should validate its use in this setting. We also assume uniform action costs, leaving a full model of cost and value, along with stronger methods for each subproblem, to future work. We hope our formalism and evaluation give these efforts a common foundation.

\bibliographystyle{ieeetr}
\bibliography{references}

\appendices
\section{Misalignment between offline and closed-loop evaluation}
\label{app:misalignment}

Offline evaluation scores a robot policy by replaying a fixed record of human behavior and charging each robot action against that static trajectory. This rests on two sequential assumptions. First, that the robot never leaves the distribution of states seen in the data. Second, that the human's goals are fixed, so the robot is only ever evaluated against the goals it was recorded alongside. Both assumptions fail once the loop is closed, as both the human and the robot co-adapt. Relaxing each assumption causes a gap between offline and closed-loop performance.

\subsection{Relaxation 1: Drift from Robot Errors Compounds Cost}
\label{sec:relaxation_drift}
The first relaxation concerns the human's response to robot mistakes. An offline evaluation maintains the state progression of the human-only baseline. In reality the state drifts from robot actions, making it more prone to further mistakes. The consequence is a compounding penalty that grows, similar to the drift in imitation learning~\citeapp{dagger}.

We model a rollout of horizon $T$. At each step the robot takes an erroneous action that deviates the environment state from the intended one, with probability $\epsilon_t$. An error pushes the robot further off its training distribution, raising its future error rate; we assume that given the rate $\epsilon$, the occurrence of errors is independent with a Bernoulli distribution, and this degradation is additive, so that the degraded error rate is $\epsilon_t + k$, saturating at 1.
The robot is fully degraded, resulting in an error at every step, once $\epsilon_t$ reaches $1$, which requires $S^\text{deg}$ errors.
\begin{equation}
S^\text{deg} = \frac{1 - \epsilon_0}{k}
\end{equation}

\begin{theorem}
\label{thm:drift}
For any $\delta \in (0,1)$, with probability at least $1 - \delta$ the robot errs at every step from time $t^{deg}$ onward, where
\begin{equation}
t^\text{deg} = \left(\frac{\sqrt{L/2} + \sqrt{L/2 + 4 \epsilon_0 S^\text{deg}}}{2 \epsilon_0 }\right)^2 .
\end{equation}
where $L = \ln\tfrac{1}{\delta}$, and $S^\text{deg} = (1-\epsilon_0)/k$
\end{theorem}

\begin{proof}
The error rate $\epsilon_t$ follows a Bernoulli distribution, starting at $\epsilon_0$ and degrading if the robot results in an error. The rate of error is therefore at least $\epsilon_0$, and we model it as $Bernoulli (\epsilon_0)$. 

Let $X_1, \dots, X_t$ be independent random variables with $X_i \in [0,1]$, and let $S_t = \sum_{i=1}^{n} X_i$, with mean $\mathbb{E}[S_t]$.
In particular, if each $X_i \sim \mathrm{Bernoulli}(\epsilon_0)$ independently, then $\mathbb{E}[S_t] = \epsilon_0 t$. Based on the one-sided Hoeffding bound, for any $\lambda > 0$,
\begin{equation}
\Pr\!\big(S_t \le \mathbb{E}[S_t] - \lambda\big) \;\le\; \exp\!\left(-\frac{2\lambda^2}{t}\right).
\end{equation}

The robot is fully degraded once $S_t \geq S^\text{deg} = (1-\epsilon_0)/k$ errors have accumulated.

Placing $\mathbb{E}[S_t] = \epsilon_0 t$ in the one-sided Hoeffding bound,
\begin{equation}
\Pr\!\big(S_t \le \epsilon_0 t - \lambda\big) \;\le\; \exp\!\left(-\tfrac{2\lambda^2}{t}\right), \forall \lambda > 0
\end{equation}

At $t=\frac{S^{deg}+\lambda}{\epsilon_0}$, and with $\lambda = \sqrt{\tfrac{t}{2}ln(1/\delta)}$

\begin{equation}
\Pr\!\big(S_t \le S^{deg}) \;\le\; \delta
\end{equation}

Thus the probability of not degrading is at most $\delta$, or equivalently, the probability of degradation is at least $1-\delta$, under

\begin{equation}
t=\frac{S^{deg}+\lambda}{\epsilon_0},~ \lambda = \sqrt{\tfrac{t}{2}ln(1/\delta)}
\end{equation}

Substituting and solving for $t$ yields, 

\begin{equation}
t^\text{deg} = \left(\frac{\sqrt{L/2} + \sqrt{L/2 + 4 \epsilon_0 S^\text{deg}}}{2 \epsilon_0}\right)^2
\end{equation}
where $L = \ln\tfrac{1}{\delta}$, and $S^\text{deg} = (1-\epsilon_0)/k$

\end{proof}


\subsection{Relaxation 2: Goal Adaptation Shifts the Optimal Policy}
\label{sec:relaxation_goals}
The second relaxation concerns the human's goals. Offline evaluation optimizes the robot against the goals recorded in the data, but a human freed from one task by the robot could reallocate to a new goal if it produces a higher net value. For example, a really busy user might always make a quick breakfast of scrambled eggs, toast and hash browns. But if the robot takes over making toast and hash browns, then the user might find time to make an omelet with vegetables instead. A robot optimized against the recorded goals cannot serve these reallocated goals, so it captures less value than an optimal robot that anticipates the adaptation.

Let $V(\pi^r, g)$ be the joint value of robot policy $\pi^r$ under human goal distribution $g$. Let $g_0$ be the goal distribution recorded in the data, and let $H(\pi^r)$ be the goal distribution the human forms in response to robot policy $\pi^r$. We assume $c_r = 0$, hence the joint objective $V - C_h$ equals the human's net value, which the human maximizes.

Offline evaluation selects the policy $\pi^\star_{\text{off}} = \arg\max_{\pi^r} V(\pi^r, g_0)$, the best response to the fixed recorded goals. The closed-loop optimum is $\pi^\star = \arg\max_{\pi^r} V\big(\pi^r, H(\pi^r)\big)$, the best response accounting for the goals the robot induces.

\begin{theorem}
\label{thm:goals}
The net value achieved by the optimal adaptive policy, i.e. one that accounts for the shifted goal distribution is at least as high as the net value achieved by the optimal non-adaptive policy
\begin{equation}
\max_{\pi^r} V\big(\pi^r, H(\pi^r)\big) \;\ge\; V\big(\pi^\star_{\text{off}}, H(\pi^\star_{\text{off}})\big),
\end{equation}
with strict inequality whenever the human's value-maximizing adaptation reaches goals that $\pi^\star_{\text{off}}$ does not serve optimally.
\end{theorem}

\begin{proof}
The policy $\pi^\star_{\text{off}}$ is one feasible policy for the coupled objective $V\big(\pi^r, H(\pi^r)\big)$, so the maximum over all policies is at least its value, giving the inequality. For strictness, consider a rollout in which the robot's assistance frees the human from a recorded goal in $g_0$, and the human reallocates to a goal of higher net value, so $H(\pi^\star_{\text{off}}) \neq g_0$. Since $\pi^\star_{\text{off}}$ was optimized against $g_0$ rather than $H(\pi^\star_{\text{off}})$, there exists a policy that serves the reallocated goals and attains strictly higher value under $H$, so the inequality is strict.
\end{proof}

\begin{corollary}
\label{cor:goals}
If the human's goals do not adapt to the robot, so that $H(\pi^r) \equiv g_0$, then $\pi^\star_{\text{off}}$ is closed-loop optimal and offline evaluation is exact. The value gap is nonzero precisely when the human adapts, that is when $H(\pi^\star_{\text{off}}) \neq g_0$.
\end{corollary}

The offline policy is suboptimal not because adaptation harms the joint objective, but because it creates value that only a robot anticipating the adaptation can capture. This is the same decoupling suboptimality established for cooperative inverse reinforcement learning, where best-responding to a fixed human demonstration is provably not part of an optimal joint policy~\cite{hadfield2016cooperative}.

\subsection{Conclusion}

The two relaxations isolate two ways, which can be stacked on top of each other, in which offline evaluation misrepresents closed-loop performance. Relaxation 1 shows that offline evaluation understates cost, because it never evaluates over the degraded regime. Relaxation 2 shows that offline evaluation wrongly evaluates value, whenever the human adapts toward higher-value goals that a fixed-goal policy cannot serve. A faithful evaluation must therefore close the loop, allowing the human both to leave robot-induced changes in place and to adapt their goals in response to the robot. Our evaluation framework includes a human model that adapts to robot actions, fully addressing the first assumption, and partially addressing the second assumption, since the LLM human can work towards future goals early but is not allowed to introduce new goals. We leave a more involved consideration of the effect of robot actions on human goals to future work, as it requires studying human behavior adaptation in more detail.
\section{Formalism: Extended Discussion}

We introduced a formalism for proactive robot assistance to tie together prior work in the area. This appendix expands on it in three ways. First, where the main text outlines the sub-problems, here we survey the existing work relevant to each in detail. Second, we draw parallels between our formulation of proactivity and value alignment in AI. Finally, we lay out extensions to the formalism and the new problems they raise.

\subsection{Prior Work Across the Sub-Problems}
\label{app:lit}

The main paper decomposes unprompted proactive assistance into three groups of sub-problems, grounding, anticipation, and planning (Section~\ref{sec:sub_problems_framework}, Table~\ref{tab:open-problems}). Here we survey prior work relevant to each. No single method addresses proactivity as a whole, so many of these references come from adjacent applications, and some are the closest available starting points rather than direct solutions. Grounding builds a goal vocabulary and labels observed behavior against it, anticipation learns the user's policy and predicts current and future goals, and planning turns those predictions into cost-minimizing robot actions. 

\subsubsection{Grounding}
Grounding involves understanding the intentional structure in user behavior from action observations.

\noindent\textbf{G1: Goal-space representation.}
Grounding a proactive robot begins with the space of goals it reasons over, which the robot must curate itself rather than receive as input. To our knowledge, no prior work autonomously curates such a space, but related efforts study the kinds of representations it could take. Hierarchical representations capture goals at multiple granularities~\citeapp{Levy2017LearningMH,Zhu2021BottomUpSD}, which matters because goals are compositional and range from a single pick-and-place to an entire chore. Open-set representations grounded in natural language~\cite{brohan2023can}, \citeapp{Zhou2023BridgingLA} accommodate an open, growing vocabulary of goals and the ambiguous line between a task variant and a distinct task. Neither line addresses the curation problem directly, and both leave open how to keep a space coarse enough to model intent yet fine enough to separate variations, or how to represent a multi-modal belief over concurrent goals when users multitask.

\noindent\textbf{G2: Action-goal association.}
Given a goal space, the robot must match observed actions to goals from a continuous stream of passive observations, the core difficulty being that users interleave actions across concurrent goals while performing extraneous ones. Prior work infers this intentional structure as a prediction of goals given actions, through symbolic models~\cite{puig2020watch,Baker2009ActionUAA} and learned models~\cite{Nakahashi2015ModelingHU,ZhiXuan2020OnlineBG}. Complementary techniques can be borrowed from hindsight experience replay~\citeapp{Andrychowicz2017HindsightER}, originally used in reinforcement learning to relabel past trajectories with sensible goals and since applied to robotics~\citeapp{Chebotar2021ActionableMU}. These methods recover associations for observed trajectories but are not designed for the open, unlabeled goal spaces a proactive robot curates.

\subsubsection{Anticipation}

The grounding sub-problems can be solved in hindsight from historical observations; once learned, the resulting models support online inference and anticipation of the goals to assist with.

\noindent\textbf{A1: Human policy learning.}
Recovering the user's goal-oriented policy $\pi^h$ from goal-conditioned actions is the province of learning from demonstration~\citeapp{ravichandar2020recent}, and specifically the subset that learns from observation alone~\citeapp{Xiong2021LearningBWA,Torabi2018BehavioralCFA}, since a proactive robot sees only passive behavior. Beyond the observation-only constraint, learning from a single user's sparse data is a challenge. Some methods reduce this by learning a discrete set of policies through user clustering or mixture models~\cite{Nikolaidis2014EfficientML,Jayanthi_strategy_2022}, or by guiding users toward more effective demonstrations~\citeapp{Sakr2025HowCE,schrum2023reciprocal}, though these rely on explicit demonstrations in the robot's actuation space rather than passive observation.

\noindent\textbf{A2: Goal inference.}
With models of the goal space and the human policy, the robot can infer the user's current goal from recent observations. Prior work casts this as maintaining a belief over goals with actions as likelihoods, addressed through symbolic methods~\cite{Ramrez2010ProbabilisticPR, puig_nopa_2023}, \citeapp{Pereira2019LandmarkBasedAF}, learned methods~\citeapp{Amado2018LSTMBasedGRA}, \cite{zhu2026proact}, and hybrids of the two~\cite{Lynch2019LearningLPA}, \citeapp{Amado2018GoalRIA}. These methods infer goals only for tasks the user has already begun.

\noindent\textbf{A3: Goal anticipation.}
Anticipating goals the user has not yet started requires modeling the temporal goal generation process. Some works use language models to predict future subgoals given a higher-level task inferred from observations~\citeapp{arora_anticipate_2024}, but are limited to subgoals of already-initiated tasks. Others treat anticipation as temporal sequence prediction, extracting habitual patterns from a history of observed actions~\cite{patel_stot_2023,patel_slatepro_2023,bartoli_streak_2025}. A complementary set of work models object movement with probabilistic filtering over how long items persist in and return to locations~\citeapp{zeng2020semantic},  \cite{saavedra2026perceptua}, which learns efficiently but cannot capture the complex patterns that learned temporal models can. A limitation cutting across this group is that the robot's own actions influence which goals the user later forms, an effect the user's goal generation process absorbs through the environment and hidden state. Because the robot cannot observe this until it starts acting, these methods face a goal distribution that shifts as the user adapts, a problem no prior work addresses.

\subsubsection{Planning}
Planning involves predicting a robot action to aid with the anticipated variable.

\noindent\textbf{P1: Cost modeling.}
To minimize the user's cost, the robot must model it, and a human action's cost is a complex function of what both agents do and of the user's preferences. The robot can reduce physical cost by taking over or supporting an action, but cost functions vary across actions and users with factors such as the enjoyment a user draws from an action and their physical capability. Interaction carries its own costs: the communication cost of answering questions and giving instructions, the teaching cost of demonstrating a task, and the monitoring cost of supervising an untrusted robot. Existing work models only a subset of these. Some approaches model action cost for optimal allocation via task complexity~\cite{malik2019complexity} and user fatigue~\cite{smith2020assessing}, or learn personalized assistive preferences from verbal feedback~\cite{patel2025taaco,dogan_grace_2025}. Robot actions also affect the user beyond the task it takes over, since sharing a space lets an action interfere with what the user does independently, for instance getting in the way of a hurried user or interrupting a call, which requires modeling social norms to avoid. Attempts at such contextual reasoning rely on simplistic state representations, hand-designed features~\cite{patel2025taaco} or scene captions~\cite{dogan_grace_2025}, that cannot capture complex social effects, while work on interruptibility~\cite{banerjee2018robot}, \citeapp{fogarty2005predicting} can inform the cost of interactive actions.

\noindent\textbf{P2: Coordinated action selection.}
Given the cost model and the anticipated goals and human policy, the robot optimizes its own policy to minimize cost. Cooperative inverse reinforcement learning~\cite{hadfield2016cooperative}, \citeapp{assistanceZero,zhao_learning_nodate} offers one view of optimizing a collaborative policy against a learned human policy, but is confined to simulated domains where rewards are available for training. Proactivity requires extending these methods to learned cost models, which may be less robust to unseen actions, and to richer behaviors: proactive communication~\cite{patel2025adapt}, hedging under uncertainty~\citeapp{ju2008design}, and distinguishing reversible from irreversible or wasteful actions. The required temporal accuracy of $\pi^h$ also varies with the scenario, since assisting a current goal demands fine-grained coordination while preparing for a not-yet-active future goal tolerates a coarser model, making the degree of collaboration a fluid variable over time rather than the binary constant robotics typically assumes.

\subsection{Parallels between Value Alignment and Proactivity}

Proactive assistance is fundamentally a problem of value alignment. By removing the task as an intermediate measure of an action's usefulness, proactivity makes alignment with the user's overall value the central objective rather than a secondary constraint. In contrast with standard alignment methods, such as RLHF, which compare alternative ways of carrying out an already-assigned task, a proactive robot must predict the task itself, reasoning from the user's values about what should be done, how, when, and by whom. This is harder, since the robot anticipates a task from scratch rather than choosing among options. However, embodiment also helps, since the robot acts within a far narrower sphere than a general-purpose chatbot. A robot situated in a specific home can model its users' values with high fidelity, spanning the universal (e.g., morality), the cultural (e.g., vegetarianism), the household (e.g., shoes on or off), and individual preference.

\subsection{Extensions of the formulation}

The formalism discussed in Sec.~\ref{sec:problem_formulation} assumes a fully observable setting with a single user. In this section, we relax these two assumptions and outline the additional challenges that arise as a result.

First, the state and the other agent's actions may be \textbf{partially observable}, so the two agents' beliefs can diverge. Each forms goals from its own belief over the state-action sequence and its own model of values and costs. For a single user we treat their value and cost model as ground truth, so even when beliefs and goals differ, the robot's objective remains to maximize the user's hidden net value. Either agent's belief over the state $s_t$ can depart from the truth under observability limits, shifting the goals they form. Such divergence is sometimes desirable: if the user steps to cross the road without seeing an oncoming car, the robot's goals should diverge from theirs rather than follow them. This divergence can be modeled through theory of mind and reduced through communication, with model reconciliation methods~\citeapp{chakraborti2019plan} generating explanations and active clarification methods~\citeapp{ren2023robots}, \cite{patel2025adapt} deciding what to ask. The robot's action space must therefore include communicative actions, verbal or non-verbal, such as opening a cabinet to reveal its contents or acting legibly so its intent is easy to read. Clarifications and explanations serve different intermediate objectives, correcting the robot's belief over the user's cost and value functions in one case and the user's belief over the state and the robot's actions in the other, but both emerge from the same underlying objective of value maximization.

Second, the robot may assist \textbf{multiple users} at once, and must align its model of values and costs across the whole set. This makes alignment substantially harder. Rather than learning one value function, the robot must learn each user's and resolve inconsistencies between them, as when members of a household hold conflicting preferences and the robot must maintain some hierarchy of control or ownership. The people who shape its value functions are moreover not limited to those it directly assists. Other stakeholders, including family, caregivers, medical practitioners, and the law, may impose their own constraints. A caregiver or doctor might specify values meant to protect the user's well-being, such as preserving their autonomy, and arbitrating between these and the user's own preferences is a genuine ethical problem. Fixed legal or ethical constraints, such as Asimov's laws, add further requirements, though these sit unambiguously at the top of the hierarchy.

\begin{table*}[t]
  \centering
  \caption{\small{Closed-loop performance under the scripted human, with a shifted evaluation distribution but not training distribution, compared to
  iteratively trained versions where distribution shift affects both phases}}
  \label{tab:results_dist}
  \setlength{\aboverulesep}{0pt}
  \setlength{\belowrulesep}{0pt}
  \begin{tabular}{l|rrr|rrr}
  \toprule
   & \multicolumn{3}{c|}{Distribution Shift at Test time only} & \multicolumn{3}{c}{Distribution Shift during Training (main results)} \\
  \cmidrule(lr){2-4} \cmidrule(lr){5-7}
  Method & Net Saved (\%) $\uparrow$ & Perturbed (\%) $\downarrow$ & F1 $\uparrow$ & Net Saved (\%) $\uparrow$ & Perturbed (\%) $\downarrow$ & F1
  $\uparrow$ \\
  \midrule
  SLaTe-PRO          & 24.2 $\pm$ 5.4 & 10.8 $\pm$ 3.0  & 0.345 $\pm$ 0.071 & 16.6 $\pm$ 3.0 & 11.8 $\pm$ 1.4 & 0.264 $\pm$ 0.043 \\
  STREAK             & 0.4 $\pm$ 0.6  & 66.3 $\pm$ 57.4 & 0.009 $\pm$ 0.015 & 20.3 $\pm$ 7.3 & 48.3 $\pm$ 19.0 & 0.296 $\pm$ 0.105 \\
  \ourMethod\ (ours) & 36.1 $\pm$ 1.8 & 16.0 $\pm$ 0.7  & 0.496 $\pm$ 0.022 & 30.8 $\pm$ 1.8 & 10.1 $\pm$ 0.8 & 0.452 $\pm$ 0.024 \\
  \bottomrule
  \end{tabular}
\end{table*}

\section{Supplementary Results}
\label{add:results}

While we report our headline results in the main paper, this section includes supplementary experiments excluded from the main paper due to space constraints.

\subsection{Effect of Confidence Threshold}
\label{add:results_confidence_thresh}

The probability threshold $c_{thresh}$ is a tunable parameter, which can be used to modulate the robot's conservativeness. Figure~\ref{fig:prob_thresh} shows the precision-recall curve and how the percentage of net actions saved and of perturbed actions changes as we vary $c_{thresh}$. We use $c_{thresh}=0.1$, which saves the most user actions with relatively few perturbed actions, and maximizes overall F1-score.

\begin{figure}[h]
    \centering
    \begin{subfigure}[]{0.49\linewidth}
    \includegraphics[trim=2mm 8mm 2mm 2mm, clip, width=\linewidth]{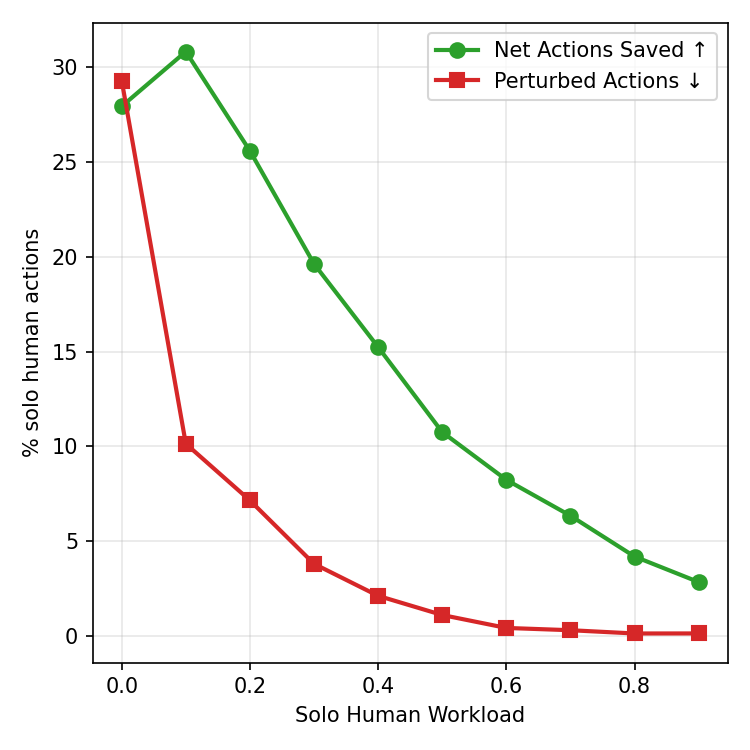}
    \caption{Net actions and perturbed actions vs. $c_{thresh}$}
    \end{subfigure}
    \begin{subfigure}[]{0.49\linewidth}
    \includegraphics[trim=2mm 2mm 2mm 2mm, clip, width=\linewidth]{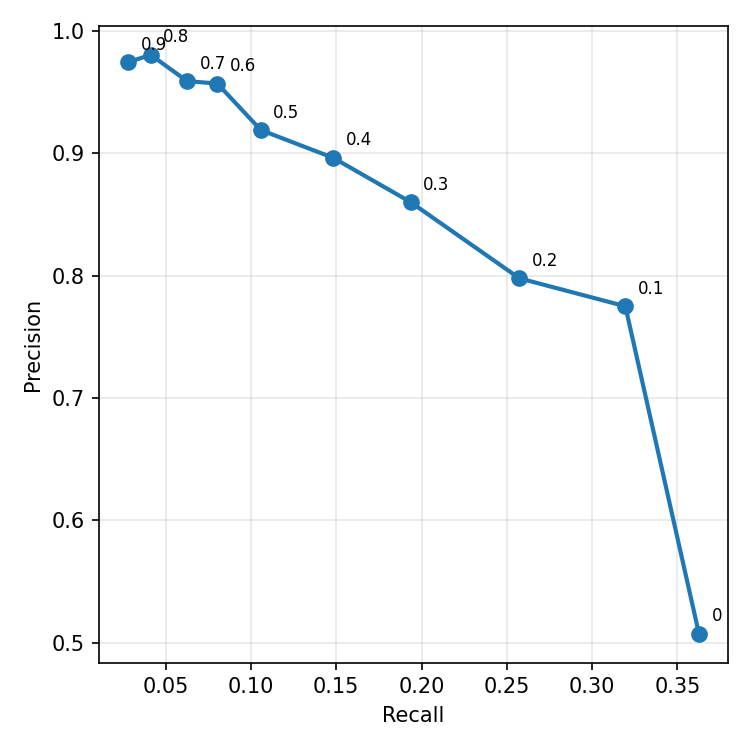}
    \caption{Precision vs. Recall}
    \end{subfigure}
    \caption{Variation in metrics as we vary $c_{thresh}$.}
    \label{fig:prob_thresh}
\end{figure}



\subsection{Effect of distribution drift during evaluation and training}

Moving from offline to closed-loop evaluation introduces state-distribution drift not only in evaluation but also in training. Here we separate the two, comparing drift that occurs during both training and testing (our main results) against drift at test time alone. To isolate the latter, the robot first observes the user without acting for all but the final five days. It therefore trains on data free of robot actions and faces distribution shift only at evaluation. Table~\ref{tab:results_dist} reports the result, isolating the performance lost to training-time drift beyond test-time drift alone. SLaTe-PRO and GAP lose 7.6\% and 5.3\% net actions saved, respectively, from the shifting training distribution. STREAK shows the opposite effect, improving when the distribution shifts during training, which we attribute to its continual-learning mechanism adapting well to a drifting training distribution. We run these experiments on the scripted user only, for cost reasons.

\section{Implementation details}

\subsection{Hyperparameters for Reproducibility}
\label{app:implementation}
This section lists the hyperparameters and fixed constants used to produce our results. The code for GAP and the evaluation setup is openly available at \url{https://github.com/Maithili/GAP}. All values are held fixed across conditions unless noted. 

\noindent\textbf{Grounding and task graphs.}
Task-graph transition probabilities use Laplace smoothing with $\alpha = 0.5$. Activity-name embeddings come from OpenAI \texttt{text-embedding-3-small} and are projected by a learned layer.

\noindent\textbf{Particle filter.}
The filter uses $N = 1000$ particles (seed $0$), at most $\mathrm{MAX\_OPEN} = 3$ concurrent open activities, and a continuation look-ahead of $\mathrm{CONT\_HOPS} = 4$ hops, resampling whenever the effective sample size falls below $N/2$. A robot action advances a matching particle marker with probability $0.5$. Open activities decay by an idle hazard $\exp(-dt\,|\mathrm{open}|/\tau)$ and are closed as stale with probability $1 - \exp(-dt/\mathrm{stale})$. The opening time-prior is Gaussian with a standard-deviation floor of $20$ min. The unclear-action probability is fixed to $\epsilon = 0.001$.

\noindent\textbf{Anticipation transformer.}
The transformer uses $d_{\mathrm{model}} = 96$, $2$ layers, $4$ heads, a feed-forward dimension of $192$, a name-projection dimension of $64$, and dropout $0.2$. We train with learning rate $10^{-3}$, a $7$-day validation split, and early stopping with patience $50$, weighting the clock-position loss by $0.5$. Times are encoded by sinusoids over periods $\{24, 18, 12, 8, 4, 2, 1\}$ h, giving $14$ dimensions, over a daily range of $6{:}00$ to $24{:}00$ ($360$ to $1440$ min) at $10$-min ticks.

\noindent\textbf{Planner and policy.}
The planner proposes candidates above a confidence threshold $c_{\mathrm{thresh}} = 0.1$, with a task-graph beam look-ahead of $3$ actions and at most $8$ actions per tick. Preparation for upcoming activities uses a $60$-min horizon, a next-activity confidence of $0.5$, and a discount of $1.0$. 
The precedence check is a soft multiplicative factor, with each candidate attenuated by every unsatisfied same-object prerequisite, driving still-pending candidates toward zero and leaving satisfied ones near one.
The per-action lead-time threshold is initialized to $\mathrm{act\_window\_min} = 30$ min and updated daily by a pinball rule toward a target revert rate of $\text{lead}_\alpha = 1/3$: it grows by $\text{lead}_s \cdot \text{lead}_\alpha$ min on a success and shrinks by $\text{lead}_s \cdot(1 - \text{lead}_\alpha)$ min on a revert or supersede, with $\text{lead}_s = 5$ min/day, clamped to $[0, 120]$ min. The acceptance estimate $P_{\mathrm{accept}}$ is Laplace-smoothed with an optimistic prior and decayed across days by a fixed factor $\mathrm{accept\_gamma} = 0.9$, with a within-day soft backoff that multiplies utility by a reject count for recently rejected placements.

\noindent\textbf{Continual learning.}
All models retrain every $n_{\mathrm{retrain}} = 5$ days on all past data, so the robot takes no actions during the first five days. Task-graph counts and the transformer's training data are recency-weighted by $\gamma$, selected at each retraining by a prequential (one-step-ahead) grid search over $\gamma \in \{0.85, 0.90, 0.95, 0.98, 1.0\}$.

\subsection{LLM Prompts}
\label{app:prompts}

Following are examples of all LLM prompts we use, including for the LLM-based curator, the human adaptation model, and the plain LLM baseline.

\subsubsection{LLM Curator Prompt}

\lstnewenvironment{prompt_curator}[1][]
{\lstset{
    basicstyle=\scriptsize\ttfamily, 
    breaklines=true,
    breakatwhitespace=true,
    breakindent=0em,
    breakautoindent=true,
    #1}},
{}
    
\begin{mdframed}[
    linewidth=1pt, 
    linecolor=black, 
    backgroundcolor=gray!10, 
    roundcorner=5pt, 
    frametitle={Example LLM Curator Prompt},
    frametitlealignment=\center, 
    everyline=true,
    nobreak=false,
]
\begin{prompt_curator}

You annotate a single person's household routine from object-relocation logs (the only sensor available - no vision, no audio). Each line is one observed move: time, object, from-location -> to-location. The day runs 06:00-24:00. Your job: partition the day's numbered events into contiguous segments, each corresponding to one activity episode (e.g. a meal, dishwashing, getting ready), and name each segment with a short snake_case activity name. REUSE an existing vocabulary name whenever the episode is the same kind of activity; create a new name only when nothing in the vocabulary fits; propose a merge when you realize two vocabulary entries denote the same underlying activity; propose a rename when an existing name no longer describes what you have observed (e.g. cereal_for_breakfast -> cereal_or_oatmeal_breakfast after seeing oatmeal days). Names are encoded by a language model downstream, so make them semantically meaningful and self-describing - say what the person is doing and with what (prefer 'wash_dinner_dishes' over 'cleanup_2').

## Current activity vocabulary (reuse these names whenever possible)
- breakfast: Morning meal: cereal or oatmeal prepared and eaten at the table.
- wash_dishes: Dishes and utensils returned from the table to the sink or dishwasher.
- computer_work: Working at the home-office desk with laptop and notebook.

## Today's observed object relocations (day 3)
0. 06:20 tooth_paste: toothbrush_holder -> bathroom_cabinet
1. 07:20 remote_control: tvstand -> sofa
2. 08:10 bowl: cupboard -> table
3. 08:10 remote_control: sofa -> tvstand
4. 08:20 food_oatmeal: cupboard -> table
5. 08:20 milk: fridge -> table
6. 08:20 spoon: cupboard -> bowl
7. 08:30 food_oatmeal: table -> cupboard
8. 08:30 milk: table -> fridge
9. 08:40 bowl: table -> sink
10. 08:40 spoon: bowl -> sink
11. 08:50 bowl: sink -> cupboard
12. 08:50 spoon: sink -> cupboard
13. 09:50 headset: filing_cabinet -> desk
14. 10:00 headset: desk -> filing_cabinet
15. 11:10 food_jam: fridge -> kitchen_counter
16. 11:10 plate: cupboard -> kitchen_counter
17. 11:20 food_bread: kitchen_counter -> kitchen_counter
18. 11:20 food_peanut_butter: cupboard -> kitchen_counter
19. 11:20 knife: knifeblock -> kitchen_counter
20. 11:20 plate: kitchen_counter -> table
21. 11:30 food_jam: kitchen_counter -> fridge
22. 11:30 food_peanut_butter: kitchen_counter -> cupboard
23. 11:40 knife: kitchen_counter -> cupboard
24. 11:40 plate: table -> cupboard
25. 13:10 remote_control: tvstand -> sofa
26. 14:00 remote_control: sofa -> tvstand
27. 14:20 cd: tvstand -> cd_player
28. 14:30 cd: cd_player -> table
29. 15:20 book: bookshelf -> sofa
30. 15:40 book: sofa -> bookshelf
31. 16:30 cookingpot: cupboard -> stove
32. 16:30 fryingpan: cupboard -> stove
33. 16:30 oil: cupboard -> kitchen_counter
34. 16:30 plate: cupboard -> kitchen_counter
35. 16:40 dry_pasta: cupboard -> kitchen_counter
36. 16:40 food_chicken: fridge -> kitchen_counter
37. 17:00 cookingpot: stove -> table
38. 17:00 dry_pasta: kitchen_counter -> cupboard
39. 17:00 food_chicken: kitchen_counter -> fridge
40. 17:00 fork: cupboard -> table
41. 17:00 fryingpan: stove -> table
42. 17:00 plate: kitchen_counter -> table
43. 17:00 spoon: cupboard -> table
44. 17:10 cookingpot: table -> sink
45. 17:10 fork: table -> sink
46. 17:10 fryingpan: table -> sink
47. 17:10 plate: table -> sink
48. 17:10 spoon: table -> sink
49. 18:00 plate: sink -> cupboard
50. 18:00 spoon: sink -> cupboard
51. 18:10 cookingpot: sink -> cupboard
52. 18:10 fryingpan: sink -> cupboard
53. 18:20 remote_control: tvstand -> sofa
54. 19:10 remote_control: sofa -> tvstand
55. 19:20 cutting_board: kitchen_counter -> table
56. 19:20 food_cheese: fridge -> cutting_board
57. 19:30 chessboard: bookshelf -> table
58. 19:30 deck_of_cards: bookshelf -> table
59. 19:30 wine_glass: cupboard -> table
60. 20:10 wine: table -> kitchen_counter
61. 20:10 wine_glass: table -> sink
62. 20:20 chessboard: table -> bookshelf
63. 20:20 deck_of_cards: table -> bookshelf
64. 20:20 food_cheese: cutting_board -> fridge
65. 20:20 wine_glass: sink -> cupboard
66. 20:40 instrument_guitar: home_office -> chair
67. 20:50 instrument_guitar: chair -> home_office

## Instructions
- Group events 0..67 into contiguous, non-overlapping segments (use first_event_idx/last_event_idx). Label every event that belongs to an activity; an isolated event that clearly interrupts an activity it does not belong to (e.g. a stray put-away in the middle of TV time) may be left out of all segments rather than absorbed into the surrounding one.
- Consecutive events that serve the same purpose belong to one segment, even with small time gaps; a long gap or a change of purpose starts a new segment.
- Name each segment with an existing vocabulary name when the episode is the same kind of activity; otherwise create a new snake_case name (declare it in new_activities with a one-line description, and set is_new_activity on those segments). Make names semantically meaningful and self-describing - they are encoded by a language model downstream.
- If two existing vocabulary entries are really the same activity, propose a merge (keep, remove, reason).
- If an existing name no longer fits what you have observed, propose a rename (old, new, new_description, reason) - e.g. cereal_for_breakfast -> cereal_or_oatmeal_breakfast. Tag today's segments with the NEW name.
\end{prompt_curator}
\end{mdframed}

\lstnewenvironment{prompt_curator_sch}[1][]
{\lstset{
    basicstyle=\scriptsize\ttfamily, 
    breaklines=true,
    breakatwhitespace=true,
    breakindent=0em,
    breakautoindent=true,
    #1}},
{}
    
\begin{mdframed}[
    linewidth=1pt, 
    linecolor=black, 
    backgroundcolor=gray!10, 
    roundcorner=5pt, 
    frametitle={Response Schema},
    frametitlealignment=\center, 
    everyline=true,
    nobreak=false,
]
\begin{prompt_curator_sch}

{
  "name": "curate_day",
  "strict": true,
  "schema": {
    "type": "object",
    "additionalProperties": false,
    "required": [
      "segments",
      "new_activities",
      "merges",
      "renames"
    ],
    "properties": {
      "segments": {
        "type": "array",
        "items": {
          "type": "object",
          "additionalProperties": false,
          "required": [
            "first_event_idx",
            "last_event_idx",
            "activity_name",
            "is_new_activity"
          ],
          "properties": {
            "first_event_idx": {
              "type": "integer"
            },
            "last_event_idx": {
              "type": "integer"
            },
            "activity_name": {
              "type": "string"
            },
            "is_new_activity": {
              "type": "boolean"
            }
          }
        }
      },
      "new_activities": {
        "type": "array",
        "items": {
          "type": "object",
          "additionalProperties": false,
          "required": [
            "name",
            "description"
          ],
          "properties": {
            "name": {
              "type": "string"
            },
            "description": {
              "type": "string"
            }
          }
        }
      },
      "merges": {
        "type": "array",
        "items": {
          "type": "object",
          "additionalProperties": false,
          "required": [
            "keep",
            "remove",
            "reason"
          ],
          "properties": {
            "keep": {
              "type": "string"
            },
            "remove": {
              "type": "string"
            },
            "reason": {
              "type": "string"
            }
          }
        }
      },
      "renames": {
        "type": "array",
        "items": {
          "type": "object",
          "additionalProperties": false,
          "required": [
            "old",
            "new",
            "new_description",
            "reason"
          ],
          "properties": {
            "old": {
              "type": "string"
            },
            "new": {
              "type": "string"
            },
            "new_description": {
              "type": "string"
            },
            "reason": {
              "type": "string"
            }
          }
        }
      }
    }
  }
}

\end{prompt_curator_sch}
\end{mdframed}

\subsubsection{LLM Human Model Prompt}

\lstnewenvironment{prompt_human}[1][]
{\lstset{
    basicstyle=\scriptsize\ttfamily, 
    breaklines=true,
    breakatwhitespace=true,
    breakindent=0em,
    breakautoindent=true,
    #1}},
{}
    
\begin{mdframed}[
    linewidth=1pt, 
    linecolor=black, 
    backgroundcolor=gray!10, 
    roundcorner=5pt, 
    frametitle={Example LLM Human Model Prompt},
    frametitlealignment=\center, 
    everyline=true,
    nobreak=false,
]
\begin{prompt_human}
You are simulating a human in a household where a robot assistant sometimes acts on its own. When the robot's action diverges from your plan, decide what the human would actually do. Adapt naturally: sometimes correct the robot, sometimes go along with what it did, sometimes skip the planned step entirely. Output one decision and a one-sentence rationale.

Current activity: breakfast
You just noticed the robot moved the laptop to sofa.
Your upcoming planned actions (next ~2 hours; this list is COMPLETE for that window):
  07:40 move laptop -> desk  <-- this object
  08:40 move notebook -> desk
  11:40 move plate -> table

Day history so far:
  06:30 human moved cereal -> table (did it as usual)
  06:35 robot moved bowl -> sink
  06:40 human moved milk -> table (did it as usual)
  07:10 human moved bowl -> sink (nothing to do, robot already handled it)

You move objects to minimize your own effort. If you'll need this object soon at a known place, it's wasteful to undo the robot and then move it again later -- instead move it straight there now, in one trip. Only undo if leaving the object where it is would be actively harmful (e.g. a perishable food left out of the fridge will spoil), or if you have no foreseeable use for it. IMPORTANT: the plan above lists EVERY move you will make in that window. If this object does not appear in it, you will not need it in that window -- do not assume a need based on the current activity or your habits; judge only from the listed actions.

Examples of how this human acted in similar situations:

Example 1:
Current activity: cooking.
You just noticed the robot moved the plate to the knifeblock.
Upcoming plan includes moving plate -> sink (in ~10 min).
Decision: undo. Rationale: The knifeblock is not where the plate belongs and you need it at the sink soon; quicker to undo now.

Example 2:
Current activity: idle.
You just noticed the robot moved the mug to the coffee_maker.
Upcoming plan has no mug actions in the next several hours.
Decision: allow. Rationale: Coffee maker is a reasonable home for the mug and you don't need it; let it be.

Example 3:
Current activity: setting_table.
You just noticed the robot moved the plate to the cupboard.
Your upcoming plan: move plate -> table (in ~20 min).
Decision: redirect. Rationale: I'll need the plate on the table shortly anyway, so I'll just carry it from the cupboard straight to the table now instead of undoing the robot and moving it twice.

Example 4:
Current activity: idle.
You just noticed the robot moved the raw_steak to the kitchen_cabinet.
Upcoming plan has no raw_steak actions for the rest of the day (you'll cook it tomorrow).
Decision: undo. Rationale: Raw steak left in a cabinet will spoil -- I need to put it back in the fridge even though it's extra work, because leaving it out is harmful.

Example 5:
Current activity: relaxing.
You just noticed the robot moved the remote_control to the tvstand.
Upcoming plan has no remote_control actions in the next 2 hours.
Decision: allow. Rationale: The tvstand is a fine spot for the remote and I won't need to put it anywhere specific soon, so I'll leave it.

Decide: 'redirect' (move it now to desk -- the place you'll need it next, saving a second trip), 'undo' (move it back where it belongs -- only if leaving it is harmful, e.g. it will spoil), or 'allow' (leave it where the robot put it).
\end{prompt_human}
\end{mdframed}

\lstnewenvironment{prompt_human_sch}[1][]
{\lstset{
    basicstyle=\scriptsize\ttfamily, 
    breaklines=true,
    breakatwhitespace=true,
    breakindent=0em,
    breakautoindent=true,
    #1}},
{}
    
\begin{mdframed}[
    linewidth=1pt, 
    linecolor=black, 
    backgroundcolor=gray!10, 
    roundcorner=5pt, 
    frametitle={Response Schema},
    frametitlealignment=\center, 
    everyline=true,
    nobreak=false,
]
\begin{prompt_human_sch}
{
  "name": "human_decision_observe_redirect",
  "strict": true,
  "schema": {
    "type": "object",
    "additionalProperties": false,
    "required": [
      "decision",
      "rationale"
    ],
    "properties": {
      "decision": {
        "type": "string",
        "enum": [
          "undo",
          "allow",
          "redirect"
        ]
      },
      "rationale": {
        "type": "string"
      }
    }
  }
}
\end{prompt_human_sch}
\end{mdframed}

\subsubsection{Plain LLM Baseline Prompt}

\lstnewenvironment{prompt_plain}[1][]
{\lstset{
    basicstyle=\scriptsize\ttfamily, 
    breaklines=true,
    breakatwhitespace=true,
    breakindent=0em,
    breakautoindent=true,
    #1}},
{}
    
\begin{mdframed}[
    linewidth=1pt, 
    linecolor=black, 
    backgroundcolor=gray!10, 
    roundcorner=5pt, 
    frametitle={Example Plain LLM Baseline Prompt},
    frametitlealignment=\center, 
    everyline=true,
    nobreak=false,
]
\begin{prompt_plain}
You are a proactive household robot living with one person. You watch where they move objects each day. Learn their routine from the past days of observations, then - given today's observations so far - predict which objects they will need moved next and where, so the robot can pre-place them. Predict only moves that genuinely help the person's routine right now or very soon.

## Past 2 days of observations (learn the person's routine from these)
### Day -2
0. 06:20 tooth_paste: toothbrush_holder -> bathroom_cabinet
1. 07:20 remote_control: tvstand -> sofa
2. 08:10 bowl: cupboard -> table
3. 08:10 remote_control: sofa -> tvstand
4. 08:20 coffee_filter: cupboard -> coffe_maker
5. 08:20 food_cereal: cupboard -> table
6. 08:20 ground_coffee: cupboard -> kitchen_counter
7. 08:20 milk: fridge -> table
8. 08:20 mug: cupboard -> coffe_maker
9. 08:20 spoon: cupboard -> bowl
10. 08:30 coffee_filter: coffe_maker -> sink
11. 08:30 ground_coffee: kitchen_counter -> cupboard
12. 08:30 mug: coffe_maker -> table
13. 08:40 food_cereal: table -> cupboard
14. 08:40 milk: table -> fridge
15. 08:50 bowl: table -> sink
16. 08:50 mug: table -> sink
17. 08:50 spoon: bowl -> sink
18. 09:00 bowl: sink -> cupboard
19. 09:00 spoon: sink -> cupboard
20. 09:10 coffee_filter: sink -> cupboard
21. 09:10 mug: sink -> cupboard
22. 10:10 headset: filing_cabinet -> desk
23. 10:20 headset: desk -> filing_cabinet
24. 11:10 food_bread: kitchen_counter -> kitchen_counter
25. 11:10 food_cheese: fridge -> kitchen_counter
26. 11:10 plate: cupboard -> kitchen_counter
27. 11:20 knife: knifeblock -> kitchen_counter
28. 11:20 plate: kitchen_counter -> table
29. 11:30 food_cheese: kitchen_counter -> fridge
30. 11:30 knife: kitchen_counter -> sink
31. 11:30 plate: table -> cupboard
32. 11:40 knife: sink -> cupboard
33. 13:00 remote_control: tvstand -> sofa
34. 13:50 remote_control: sofa -> tvstand
35. 14:10 cd: tvstand -> cd_player
36. 14:30 cd: cd_player -> table
37. 15:50 book: bookshelf -> sofa
38. 16:00 book: sofa -> bookshelf
39. 17:10 cookingpot: cupboard -> stove
40. 17:10 dry_pasta: cupboard -> kitchen_counter
41. 17:10 fryingpan: cupboard -> stove
42. 17:10 oil: cupboard -> kitchen_counter
43. 17:10 plate: cupboard -> kitchen_counter
44. 17:20 food_chicken: fridge -> kitchen_counter
45. 17:30 cookingpot: stove -> table
46. 17:30 dry_pasta: kitchen_counter -> cupboard
47. 17:30 food_chicken: kitchen_counter -> fridge
48. 17:30 fryingpan: stove -> table
49. 17:30 plate: kitchen_counter -> table
50. 17:30 spoon: cupboard -> cookingpot
51. 17:40 fork: cupboard -> sink
52. 17:40 spoon: cookingpot -> sink
53. 17:50 cookingpot: table -> sink
54. 17:50 fryingpan: table -> sink
55. 17:50 plate: table -> sink
56. 18:30 plate: sink -> cupboard
57. 18:40 cookingpot: sink -> cupboard
58. 18:40 fryingpan: sink -> cupboard
59. 18:40 spoon: sink -> cupboard
60. 19:10 chessboard: bookshelf -> table
61. 19:10 deck_of_cards: bookshelf -> table
62. 19:10 food_donut: kitchen_counter -> kitchen_counter
63. 19:40 food_donut: kitchen_counter -> trashbag
64. 19:50 chessboard: table -> bookshelf
65. 19:50 deck_of_cards: table -> bookshelf
66. 20:00 remote_control: tvstand -> sofa
67. 20:50 remote_control: sofa -> tvstand

### Day -1
0. 06:30 tooth_paste: toothbrush_holder -> bathroom_cabinet
1. 07:30 remote_control: tvstand -> sofa
2. 08:20 bowl: cupboard -> table
3. 08:20 remote_control: sofa -> tvstand
4. 08:30 cup: cupboard -> table
5. 08:30 food_oatmeal: cupboard -> table
6. 08:30 juice: cupboard -> table
7. 08:30 milk: fridge -> table
8. 08:40 spoon: cupboard -> bowl
9. 08:50 bowl: table -> sink
10. 08:50 food_oatmeal: table -> cupboard
11. 08:50 juice: table -> cupboard
12. 08:50 milk: table -> fridge
13. 08:50 spoon: bowl -> sink
14. 09:00 bowl: sink -> cupboard
15. 09:00 spoon: sink -> cupboard
16. 09:10 cup: table -> cupboard
17. 11:10 food_bread: kitchen_counter -> kitchen_counter
18. 11:10 food_cheese: fridge -> kitchen_counter
19. 11:10 plate: cupboard -> kitchen_counter
20. 11:20 knife: knifeblock -> kitchen_counter
21. 11:20 plate: kitchen_counter -> table
22. 11:30 food_cheese: kitchen_counter -> fridge
23. 11:30 knife: kitchen_counter -> sink
24. 11:30 plate: table -> sink
25. 11:40 knife: sink -> cupboard
26. 11:40 plate: sink -> cupboard
27. 14:10 radio: tvstand -> dining_room
28. 15:30 book: bookshelf -> chair
29. 15:40 book: chair -> bookshelf
30. 17:10 cookingpot: cupboard -> stove
31. 17:10 fryingpan: cupboard -> stove
32. 17:10 oil: cupboard -> kitchen_counter
33. 17:10 plate: cupboard -> kitchen_counter
34. 17:20 food_rice: cupboard -> kitchen_counter
35. 17:20 food_vegetable: fridge -> kitchen_counter
36. 17:30 cookingpot: stove -> table
37. 17:30 food_rice: kitchen_counter -> cupboard
38. 17:30 food_vegetable: kitchen_counter -> fridge
39. 17:30 fork: cupboard -> table
40. 17:30 fryingpan: stove -> table
41. 17:30 plate: kitchen_counter -> table
42. 17:30 spoon: cupboard -> table
43. 17:40 cookingpot: table -> sink
44. 17:40 fork: table -> sink
45. 17:40 fryingpan: table -> sink
46. 17:40 plate: table -> sink
47. 17:40 spoon: table -> sink
48. 18:00 chessboard: bookshelf -> table
49. 18:00 deck_of_cards: bookshelf -> table
50. 18:00 tea: cupboard -> table
51. 18:10 cup: cupboard -> table
52. 18:50 cup: table -> cupboard
53. 18:50 tea: table -> cupboard
54. 19:00 chessboard: table -> bookshelf
55. 19:00 deck_of_cards: table -> bookshelf
56. 19:40 plate: sink -> cupboard
57. 19:50 cookingpot: sink -> cupboard
58. 19:50 fryingpan: sink -> cupboard
59. 19:50 spoon: sink -> cupboard
60. 20:40 remote_control: tvstand -> sofa
61. 21:30 remote_control: sofa -> tvstand

## Today so far (incomplete; current time is 08:40)
0. 06:20 tooth_paste: toothbrush_holder -> bathroom_cabinet
1. 07:20 remote_control: tvstand -> sofa
2. 08:10 bowl: cupboard -> table
3. 08:10 remote_control: sofa -> tvstand
4. 08:20 food_oatmeal: cupboard -> table
5. 08:20 milk: fridge -> table

## Predict up to 10 object relocations the robot should do RIGHT NOW (object and destination), based on what the person will need next in their routine.
\end{prompt_plain}
\end{mdframed}

\lstnewenvironment{prompt_plain_sch}[1][]
{\lstset{
    basicstyle=\scriptsize\ttfamily, 
    breaklines=true,
    breakatwhitespace=true,
    breakindent=0em,
    breakautoindent=true,
    #1}},
{}
    
\begin{mdframed}[
    linewidth=1pt, 
    linecolor=black, 
    backgroundcolor=gray!10, 
    roundcorner=5pt, 
    frametitle={Response Schema},
    frametitlealignment=\center, 
    everyline=true,
    nobreak=false,
]
\begin{prompt_plain_sch}
{
  "name": "predict_relocations",
  "strict": true,
  "schema": {
    "type": "object",
    "additionalProperties": false,
    "required": [
      "predictions"
    ],
    "properties": {
      "predictions": {
        "type": "array",
        "items": {
          "type": "object",
          "additionalProperties": false,
          "required": [
            "object",
            "destination"
          ],
          "properties": {
            "object": {
              "type": "string",
              "enum": [
                "book",
                "bowl",
                "cd",
                "cd_player",
                "chessboard",
                "cleaning_solution",
                "... 61 total"
              ]
            },
            "destination": {
              "type": "string",
              "enum": [
                "bathroom",
                "bathroom_cabinet",
                "bathroom_counter",
                "bedroom",
                "book",
                "bookshelf",
                "... 85 total"
              ]
            }
          }
        }
      }
    }
  }
}
\end{prompt_plain_sch}
\end{mdframed}

\makeatletter
\let\@origsection\section
\def\section*#1{\let\section\@origsection\@origsection*{Appendix References}}
\makeatother
\bibliographystyleapp{ieeetr}
\bibliographyapp{references}

\end{document}